\documentclass{article}

\usepackage{microtype}
\usepackage{multirow}
\usepackage{graphicx}
\usepackage{subcaption}
\usepackage{booktabs} %

\usepackage{hyperref}
\usepackage{enumitem}
\usepackage{xfrac}

\usepackage[ruled,vlined]{algorithm2e}

\usepackage[accepted]{icml2026}

\usepackage{amsmath}
\usepackage{amssymb}
\usepackage{mathtools}
\usepackage{amsthm}
\usepackage{xcolor}
\usepackage{wasysym}
\usepackage{diagbox}
\usepackage{afterpage}
\everypar{\looseness=-1}
\allowdisplaybreaks
\newcommand*{\defeq}{\stackrel{\text{def}}{=}}

\usepackage[capitalize,noabbrev]{cleveref}

\theoremstyle{plain}
\newtheorem{theorem}{Theorem}[section]
\newtheorem{proposition}[theorem]{Proposition}
\newtheorem{lemma}[theorem]{Lemma}

\theoremstyle{definition}
\newtheorem{definition}[theorem]{Definition}

\theoremstyle{remark}

\newcommand{\cF}{\mathcal{F}}

\newcommand{\E}{\mathbb{E}}
\newcommand{\R}{\mathbb{R}}
\DeclareMathOperator*{\argmax}{arg\,max}

\newcommand{\Ex}[1]{{\underset{#1}{\mathbb{E}}}}

\definecolor{mysunset}{RGB}{255, 143, 1}
\newcommand{\sunset}[1]{{\color{mysunset} #1}}

\usepackage[textsize=tiny]{todonotes}

\icmltitlerunning{Variational Entropic Optimal Transport}

\begin{document}

\twocolumn[
  \icmltitle{Variational Entropic Optimal Transport}

  \icmlsetsymbol{equal}{*}

  \begin{icmlauthorlist}
    \icmlauthor{Roman Dyachenko}{hse}
    \icmlauthor{Nikita Gushchin}{aai,axxx}
    \icmlauthor{Kirill Sokolov}{msu}
    \icmlauthor{Petr Mokrov}{aai}
    \icmlauthor{Evgeny Burnaev}{aai,axxx}
    \icmlauthor{Alexander Korotin}{aai,axxx}
  \end{icmlauthorlist}

  \icmlaffiliation{hse}{Higher School of Economics, Moscow, Russia}
  \icmlaffiliation{aai}{Applied AI Institute, Moscow, Russia}
  \icmlaffiliation{axxx}{AXXX, Russia}
  \icmlaffiliation{msu}{Lomonosov Moscow State University, Moscow, Russia}

  \icmlcorrespondingauthor{Roman Dyachenko}{rrdiachenko@edu.hse.ru}
  \icmlcorrespondingauthor{Nikita Gushchin}{i.nikita.gushchin@gmail.com}

  \icmlkeywords{Machine Learning, ICML}

  \vskip 0.3in
]

\printAffiliationsAndNotice{}  %

\begin{abstract}
Entropic optimal transport (EOT) in continuous spaces with quadratic cost is a classical tool for solving the domain translation problem. In practice, recent approaches optimize a weak dual EOT objective depending on a single potential, but doing so is computationally not efficient due to the intractable log-partition term. Existing methods typically resolve this obstacle in one of two ways: by significantly restricting the transport family to obtain closed-form normalization (via Gaussian-mixture parameterizations), or by using general neural parameterizations that require simulation-based training procedures. We propose Variational Entropic Optimal Transport (VarEOT), based on an exact variational reformulation of the log-partition $\log \E[\exp(\cdot)]$ as a tractable minimization over an auxiliary log-normalizer. This yields a differentiable learning objective optimized with stochastic gradients and avoids the necessity of MCMC simulations during the training. We provide theoretical guarantees, including finite-sample generalization bounds and approximation results under universal function approximation. Experiments on synthetic data and unpaired image-to-image translation demonstrate competitive or improved translation quality, while comparisons within the solvers that use the same weak dual EOT objective support the benefit of the proposed optimization principle. The code for our solver can be found at \url{https://github.com/DrEternity/VarEOT}. 
\end{abstract}

\vspace{-3mm}
\section{Introduction}
\vspace{-1mm}
Entropic Optimal Transport (EOT) with quadratic cost is well-established mathematical framework with strong theoretical properties that has found wide application in
generative modeling and especially for unpaired domain translation. Despite this, the practical adoption of entropic transport methods has been limited by the lack of efficient and flexible algorithms. Existing approaches typically suffer from at least one of the following drawbacks: they are not amenable to simulation-free training, requiring costly sampling at each iteration \citep{mokrov2024energyguided}; they require adversarial optimization \citep{gushchin2024adversarial}; they involve training a sequence of models rather than a single objective \citep{shi2023diffusion}; they impose restrictive parametric forms on the transport plan \citep{korotin2024light}; or they turn out to be too sensitive to the entropic regularization strength \citep{daniels2021score}.

In our paper, we make a decisive step towards solving the drawbacks of existing EOT methods, and propose a novel simulation-free training solver based on an innovative variational reformulation of the weak dual EOT objective. We present the following \textbf{main contributions}:
\vspace{-2mm}

\begin{enumerate}%
    \item \textbf{Variational dual objective} (\wasyparagraph\ref{sec:variational-dual}).  We derive an equivalent reformulation of the weak dual objective of entropic OT with quadratic cost in which the intractable log-partition term is replaced by an exact variational minimization over an auxiliary log-normalizer.
    \item \textbf{Simulation-free training solver} (\wasyparagraph\ref{sec:practical-algorithm}). Building on this reformulation, we propose a \emph{simulation-free training} variational solver that jointly learns the dual potential and the auxiliary normalizer via neural parameterizations (no MCMC).
    \item \textbf{Learning guarantees} (\wasyparagraph\ref{sec:learning-guarantees}). We provide finite-sample learning guarantees for recovery of the entropic OT plan, decomposing error into estimation and approximation terms, and show vanishing approximation error under universal function approximation.
    \item \textbf{Evaluation} (\wasyparagraph\ref{sec-experiments}). We evaluate our solver on synthetic data and unpaired image-to-image translation. For the latter, we highlight gains by comparing against other solvers that optimize the same weak dual objective.
\end{enumerate}

\providecommand{\bbR}{\mathbb{R}}
\providecommand{\cX}{\mathcal{X}}
\providecommand{\cY}{\mathcal{Y}}
\providecommand{\cZ}{\mathcal{Z}}
\providecommand{\cP}{\mathcal{P}}
\providecommand{\cC}{\mathcal{C}}
\providecommand{\cPac}{\mathcal{P}_{\mathrm{ac}}}
\providecommand{\bbP}{\mathbb{P}}
\providecommand{\bbQ}{\mathbb{Q}}
\providecommand{\dv}[2]{\frac{d #1}{d #2}}

\vspace{3mm}
\section{Background}
\vspace{-1mm}
This section provides the necessary background on entropic optimal transport and its optimization using weak dual reformulation. In \wasyparagraph\ref{sec:eot-background}, we recall the entropic optimal transport problem
with quadratic cost, introduce its weak dual formulation, and describe the structure of the optimal transport plan induced by the optimal dual potential. In \wasyparagraph\ref{sec:practical-setup}, we clarify our learning setup. In \wasyparagraph\ref{sec:existing-solvers}, we review existing solvers based on weak dual objective, highlighting two main paradigms: simulation-based optimization via implicit sampling proposed by the authors of EgNOT \citep{mokrov2024energyguided} and simulation-free methods based on restrictive parametric assumptions proposed by the authors of LightSB \citep{korotin2024light}. This discussion motivates the need for a dual formulation amenable to simulation-free training while remaining sufficiently expressive, which we address in subsequent \wasyparagraph\ref{sec:method} by proposing our novel VarEOT solver. %

\vspace{-2mm}
\subsection{Entropic Optimal Transport with the quadratic cost}
\label{sec:eot-background}
\vspace{-1mm}

Let $p_0, p_1 \in \mathcal{P}_{\mathrm{ac}}(\mathbb{R}^{D})$ be the set of absolutely continuous Borel probability measures on $\mathbb{R}^{D}$. Let $\Pi(p_0,p_1)$ denote the set of couplings (transport plans) on $\mathbb{R}^D\times\mathbb{R}^D$ with marginals $p_0$ and $p_1$. We write $\pi(x_0,x_1)$ for a plan density; $H$ is the differential entropy.
The Entropic Optimal Transport (EOT) problem is given by:
\begin{align}%
\mathrm{EOT}_{\varepsilon}(p_0,p_1)
&\overset{\mathrm{def}}{=}
\min_{\pi\in\Pi(p_0,p_1)}
\Bigg\{
\underbrace{\underset{\pi(x_0, x_1)}{\mathbb{E}}\!\left[\frac{\|x_0-x_1\|^2}{2}\right]}_{\text{optimal transport term}} \nonumber \\
&\underbrace{- \varepsilon\int_{\mathbb{R}^D} \!\!H \!\big(\pi(\cdot\mid x_0)\big)p_0(x_0)dx_0}_{\text{entropic regularization term}}
\Bigg\}. \label{eq:eot}
\end{align}
The entropic regularization term in \eqref{eq:eot} is due to \citep{mokrov2024energyguided}. There are equivalent forms of the term \cite{cuturi2013sinkhorn, leonard2014survey} which differ only by additive constants that do not affect the solution. 

The EOT problem admits a unique minimizer $\pi^\ast$, referred to as the EOT plan. While the primal formulation~\eqref{eq:eot} is conceptually appealing, it is computationally inconvenient, since enforcing the marginal constraints $\pi \in \Pi(p_0,p_1)$ requires optimizing over a complex set of probability measures that does not admit a straightforward parametrization. 

\textbf{Weak dual form of EOT.} Objective \eqref{eq:eot} admits the following weak dual representation \citep[Theorem~1]{mokrov2024energyguided}:
\vspace{-1mm}
\begin{equation} \sup_{f\in L_1(\R^D)}
\bigg\{\underbrace{
\Ex{p_1(x_1)}f(x_1)
- 
\varepsilon
\Ex{p_0(x_0)}\log Z(f,x_0)}_{\defeq \mathcal{L}(f)}
\biggr\}, \!\!
\label{eq:dual-final}
\end{equation}
where the $\sup$ is taken over a function $f\in L_1(p_1)$ satisfying $f\equiv-\infty$ outside the $\mathrm{supp}(p_1)$ and
\vspace{-1mm}
\begin{equation}
Z(f,x_0)
\defeq
\int_{\mathbb{R}^{D}}
\exp\!\left(
\frac{f(x_1) - \tfrac{1}{2}\|x_0-x_1\|^{2}}{\varepsilon}
\right) dx_1
\label{eq:partition}
\end{equation}
is the partition function.

\textbf{Optimal transport plan.}
Let $f^*$ be an optimizer of \eqref{eq:dual-final}, the corresponding optimal transport plan $\pi^*$ \citep[Theorem 1]{mokrov2024energyguided} can be recovered from it. By the disintegration with respect to the source marginal $p_0$ we have:
\begin{equation}
\pi^*(x_0,x_1)  
=
\pi^*(x_1 \mid x_0)\, p_0(x_0),
\label{eq:joint_pi_star}
\end{equation}
Then the density of conditional distribution $\pi^*(\cdot \mid x_0)$ in \eqref{eq:joint_pi_star} is directly defined by $f^*$:
\vspace{-2mm}
\begin{equation}
\hspace*{-2mm}\pi^*(x_1 \!\!\mid\! x_0) \!=\! \frac{1}{Z(f^*,x_0)}
\exp\!\left(\!
\frac{f^*(x_1) \!-\! \tfrac{1}{2}\|x_0\!-\!x_1\|^{2}}{\varepsilon}
\right).
\label{eq:cond_pi_star}
\end{equation}

\vspace{-4mm}
\subsection{Computational EOT setup}\label{sec:practical-setup}
\vspace{-1mm}

In practice, the source and target distributions, $p_0, p_1$, as well as the EOT objective \eqref{eq:eot}, could be expressed and treated in different ways. To avoid possible misunderstanding, below we formalize our \textbf{practical learning setup}:

\hspace{1mm}\fbox{%
    \parbox{0.93\linewidth}{
    We assume that source and target distributions $p_0$ and $p_1$ are accessible only by a limited number of i.i.d. empirical samples (datasets) ${\{x_0^1, x_0^2, \dots x_0^{N}\} \sim p_0}$; ${\{x_1^1, x_1^2, \dots x_1^{M}\} \sim p_1}$. Our aim is to approximate the optimal conditional plan $\pi^*(\cdot \vert x_0)$ (eq. \eqref{eq:cond_pi_star}) between entire \textit{distributions} $p_0$ and $p_1$. 
    The recovered solution should provide the \textit{out-of-sample} estimation, i.e., allow generating samples from $\pi^*(\cdot \vert x_0^{\text{new}})$, where $x_0^{\text{new}}$ is a new sample from $p_0$ which is not necessarily present in the train dataset.
    }
}%

This setup falls within \textbf{continuous} OT, in contrast to discrete OT \citep{cuturi2013sinkhorn,peyre2019computational}, which is designed to compute one-to-one or one-to-many \textit{correspondence} directly between the collections of provided source and target samples. As a result, discrete OT approaches do not naturally accommodate the out-of-sample estimation demanded by continuous OT. In our manuscript, we focus exclusively on continuous OT approaches, treating discrete OT as a considerably different direction.

\vspace{-2mm}
\subsection{Existing Weak Dual Formulation Solvers}
\label{sec:existing-solvers}
\vspace{-1mm}

The practical optimization of problem \eqref{eq:dual-final} remains challenging due to the presence of the partition function $Z(f,x_0)$, which is, in general, intractable to compute exactly. Below we review two representative strategies to optimize the semi-dual objective: (i) general neural potentials with implicit sampling (EgNOT), and (ii) simulation-free objectives enabled by restrictive parametric transport families (LightSB).

\textbf{EgNOT solver.} The authors of EgNOT~\citep{mokrov2024energyguided} solve \eqref{eq:dual-final} by parametrizing $f_{\theta}$ by a neural network and deriving the gradient of the weak dual objective $\mathcal{L}(f_{\theta})$:
\vspace{-1mm}
\begin{equation}
\begin{aligned}
\nabla_{\theta}\,\mathcal{L}( f_\theta)
=
- \Ex{p_0(x_0)}\left[\Ex{\pi_{\theta}(x_1|x_0)}\left[ \nabla_{\theta} f_\theta(x_1)\right]\right]
\, \\
+
\Ex{p_1(x_1)}\left[\nabla_{\theta} f_\theta(x_1)\, \right],
\end{aligned}
\end{equation}\vspace{-3mm}\\
where $\pi_{\theta}(x_1|x_0)$ is given by:
\vspace{-1mm}
\begin{equation*}
\hspace*{-2mm}\pi_{\theta}(x_1 \!\!\mid\! x_0)
\!=\!
\frac{1}{Z(f_{\theta},x_0)}
\exp\!\left(\!
\frac{f_{\theta}(x_1) \!-\! \tfrac{1}{2}\|x_0\!-\!x_1\|^{2}}{\varepsilon}
\right).
\end{equation*}
While flexible, \textbf{this approach is \emph{not simulation-free}}: each evaluation of
the loss or its gradient requires sampling from the model distribution
$\pi_{\theta}(x_1 \mid x_0)$, which is itself defined implicitly through the
neural potential $f_{\theta}$. In practice, this sampling step is carried out
using Markov chain Monte Carlo (MCMC) methods \citep{girolami2011riemann, hoffman2014no, samsonov2022localglobal}, such as Langevin dynamics.
However, MCMC-based sampling can be computationally expensive, sensitive to
hyperparameters, and slow to mix, especially in high-dimensional settings.

\textbf{LightSB solver.} An alternative strategy is proposed in LightSB~\citep{korotin2024light}. The authors introduce adjusted potential $v_{\theta}$ and parametrization:
\vspace{-1mm}
\begin{equation}
\pi_{\theta}(x_1|x_0)=\frac{\exp\big(\sfrac{\langle x_0,x_1\rangle}{\varepsilon}\big)v_{\theta}(x_1)}{c_{\theta}(x_0)},
\label{eq:lightsb-plan-parametric-full}
\end{equation}\vspace{-3mm}\\
where, $c_{\theta}(x_{0})\defeq\int_{\mathbb{R}^{D}}\exp\big(\sfrac{\langle x_0,x_1\rangle}{\varepsilon}\big)v_{\theta}(x_1)dx_1$ is the normalization. They then consider the optimization problem::
\vspace{-2mm}
\begin{equation}
\min_{v_{\theta}} \biggr\{ \Ex{p_0(x_0)}\log c_{\theta}(x_0)-\Ex{p_1(x_1)} \log v_{\theta}(x_1)\biggr\}.
\label{eq:lightsb-main-obj-feasible}
\end{equation}\vspace{-3mm}\\
This problem is equivalent to the problem \eqref{eq:dual-final} under a different parametrization, specifically:
\vspace{-1mm}
\begin{gather}
    v_{\theta}(x_1) = \exp(-\frac{\|x_1\|^{2}}{2\varepsilon})\exp(\frac{f_{\theta}(x_1)}{\varepsilon});
    \nonumber
    \\
    c_{\theta}(x_0) = \exp(\frac{\|x_0\|^{2}}{2\varepsilon})Z(f_{\theta}, x_0).
    \nonumber
\end{gather}\vspace{-3mm}\\
To solve the problem in practice, the authors of LightSB circumvent the intractability of normalization $c_{\theta}(x_0)$ (equivalent of $Z(f,x_0)$, eq. \eqref{eq:partition}) by directly parameterizing the potential $v_{\theta}(x_1)$ in form of a Gaussian mixture:
\begin{gather}
    v_{\theta}(x_1) =\sum_{k=1}^{K'} \alpha_k \,\mathcal{N}(x_1\vert r_k,\varepsilon S_k),
\end{gather}
where $\theta\defeq\{\alpha_{k},r_{k},S_{k}\}_{k=1}^{K'}$ are the parameters: $\alpha_{k}\geq 0$, $r_k\in\mathbb{R}^{D}$ and symmetric $0\!\prec\! S_{k}\!\in\!\mathbb{R}^{D\times D}$. Such a parameterization restricts conditional density $\pi(x_1 \!\mid\! x_0)$ to form:
\begin{gather}
    \pi_{\theta}(x_{1}|x_0)=\frac{1}{c_{\theta}(x_0)}\sum_{k=1}^{K'}\widetilde{\alpha}_{k}(x_0)\mathcal{N}(x_{1}|r_{k}(x_0),\varepsilon S_{k}),
    \nonumber
\end{gather}
where $\widetilde{\alpha}(x_0)$ and $r_k(x_0)$ are some functions of $\{\alpha_{k},r_{k},S_{k}\}_{k=1}^{K'}$ and $c_{\theta}(x_0)$ is given in a closed form:
\begin{gather}
    c_{\theta}(x_0) = \sum_{k=1}^{K'}\widetilde{\alpha}_k(x_0).
    \nonumber
\end{gather}
While this parameterization leads to a fully tractable and simulation-free objective, it \textbf{restricts the expressiveness of the method}, as it limits admissible transport plans to a rather narrow parametric family.

\textbf{Summary.} Existing weak dual solvers trade off between expressiveness and tractability: EgNOT supports flexible potentials but requires MCMC during training, while LightSB is simulation-free but restricts the conditional plan family. Below, we present our novel method, VarEOT, which takes the \textbf{best of two worlds} by enabling simulation-free training \emph{without} restricting $\pi(x_1\!\mid\! x_0)$ to a narrow parametric family.

\vspace{3mm}
\section{Variational Entropic Optimal Transport}\label{sec:method}
\vspace{-1mm}

In this section, we introduce our variational approach to entropic optimal transport. In \wasyparagraph\ref{sec:variational-dual}, we derive a new variational dual formulation of the EOT objective that replaces the intractable log-partition function with a tractable variational upper bound, yielding a fully differentiable and simulation-free training loss. Building on this formulation, \wasyparagraph\ref{sec:practical-algorithm}  presents a practical learning algorithm based on neural parameterization of the dual potential and the auxiliary variational function, together with a Langevin sampling procedure for generating transport samples. We further position our method relative to existing dual solvers and highlight its practical advantages. In \wasyparagraph\ref{sec:learning-guarantees} we conduct the analysis of our method from the perspectives of statistical learning theory (finite sampling learning guarantees and approximation with neural networks). All \underline{proofs} are provided in Appendix \ref{app:Proofs}. 

\vspace{-2mm}
\subsection{New Variational Dual Formulation of EOT}
\label{sec:variational-dual}
\vspace{-1mm}

Our goal is to propose a weak dual solver that, unlike EgNOT and LightSB, does not require simulation during training and allows for expressive parameterization.

A key challenge in this setting is differentiating through the partition function $\log Z(f,x_0)$: as a logarithm of an expectation, it cannot be unbiasedly estimated from finite samples in a straightforward way.
To overcome this difficulty, we adopt a variational approximation for the logarithm, which allows us to construct a tractable, differentiable surrogate for the dual objective. The details of this variational approach are formalized in the proposition below. 

\begin{proposition}[Variational bound for the partition function]
\label{lem:variational-Z}
The logarithm of partition function $\log Z(f, x_0)$
admits the variational upper bound:
\begin{equation}
\begin{aligned}
&\log Z(f,x_0)
\le
-1 + \xi(x_0) + \frac{D}{2}\log (2\pi\varepsilon) + \\
&\Ex{z\sim\mathcal{N}\!(0, I)}\bigg[
\exp\Bigl(\frac{f(x_0+\sqrt{\varepsilon}z)}{\varepsilon} - \xi(x_0)
\Bigr)\bigg],
\end{aligned}
\end{equation}
where $\xi : \bbR^D \rightarrow \bbR$ is an arbitrary integrable function. 
The upper bound is tight when
\begin{equation}\label{eq:optimalxi}
\begin{aligned}
\xi_f(x_0) = \log Z(f, x_0) - \frac{D}{2} \log(2\pi\varepsilon),
\end{aligned}
\end{equation}
i.e., at the (shifted by a constant) log partition function.
\end{proposition}

Thanks to the obtained estimate, we can obtain tractable simulation-free loss:

\begin{theorem}[Variational dual form of EOT]
\label{prop:euclidean-dual}
Let
\begin{equation}
\label{eq:final-loss}
\begin{aligned}
\mathcal{L}(f, \xi) \defeq \, &\sunset{\varepsilon \Big(1 - \frac{D}{2}\log(2 \pi \varepsilon)\Big)} + \\
\Ex{x_1\sim p_1}\left[f(x_1)\right] &-\varepsilon\Ex{x_0\sim p_0}\left[\xi(x_0)\right]
-\\
\varepsilon\!\Ex{x_0\sim p_0}\!\bigg[\Ex{z\sim\mathcal{N}\!(0, I)}&
\!\!\!\exp\Bigl(\!\frac{f(x_0+\sqrt{\varepsilon}z)}{\varepsilon} - \xi(x_0)
\!\Bigr)\bigg],
\end{aligned}
\end{equation}
Then the entropic optimal transport weak dual formulation admits the following variational form:
\begin{equation}
\mathrm{EOT}_{\varepsilon}(p_0,p_1)
=
\sup_f\mathcal{L}(f)
=
 \sup_{f, \xi} \mathcal{L}(f, \xi).
\end{equation}
The optimal solution $(f^*, \xi_f^*)$, where $\xi_f$ set by (\ref{eq:optimalxi}), recovers 
\begin{gather}
    \!\!\! \pi^*(x_1|x_0) \!=\! \frac{(2\pi \varepsilon)^{-\frac{D}{2}}}{\exp(\xi_f(x_0))}
\exp\!\left(
\frac{f^*(x_1) - \tfrac{1}{2}\|x_0-x_1\|^{2}}{\varepsilon}
\right)
\nonumber
\end{gather}

\end{theorem}
\vspace{-2mm}
This novel variational dual form allows us to overcome the original problem of estimation of log partition function $\log Z(f, x_0)$ in the weak dual form \eqref{eq:dual-final}. 

For convenience we define:
\vspace{-1mm}
\begin{equation*}
\begin{aligned}
    \pi^{f,\xi}(x_0,x_1\!) &\!\defeq\! \frac{(2\pi \varepsilon)^{-\frac{D}{2}} p_0(x_0\!)}{\exp(\xi(x_0))}\!\exp\!\bigg[\!\frac{f(x_1)\!-\!\frac{1}{2}\|x_0\!-\!x_1\|^2}{\varepsilon}\bigg]; \\
    \pi^{f}(x_0, x_1) &\defeq \pi^{f, \xi_f} (x_0, x_1),
\end{aligned}
\end{equation*}\vspace{-3mm}\\
where $\xi_f$ is due to \eqref{eq:optimalxi}.
The following theorem establishes that the VarEOT optimality gap directly corresponds to the KL discrepancy between the recovered measure $\pi^{f,\xi}$ ($\pi^f$) and the EOT plan $\pi^*$.  

\begin{theorem}\label{prop:fxiequality}
    For any $f\in L_1(p_1)$ and $\xi\in L_1(p_0)$,
\begin{equation}
\label{eq:excess-kl}
\varepsilon
    \mathrm{KL}\!\left(\pi^\ast \,\middle\|\, \pi^{ f}\right)
    \le
    \varepsilon\,\mathrm{KL}\!\left(\pi^* \,\middle\|\, \pi^{f,\xi}\right)
    =
    \mathcal{L}^* - \mathcal{L}(f,\xi),
\end{equation}
where $\mathrm{KL}(\cdot\|\cdot)$ denotes the Kullback--Leibler divergence between non-negative measures (see \underline{Definition}~\ref{def:KLdiv} in Appendix \ref{app:Proofs}), and $\mathcal{L}^*$ is the optimal value of the objective \eqref{eq:final-loss}.
\end{theorem}

Our Theorem \ref{prop:fxiequality} certifies that optimizing objective \eqref{eq:final-loss} directly enables us to approximate the ground truth EOT plan. Additionally, eq. \eqref{eq:excess-kl} suggests that at inference stage it is better to use $\pi^f$, not $\pi^{f, \xi}$, which we exploit in our practical implementation, see the next section \wasyparagraph\ref{sec:practical-algorithm}.

\vspace{-2mm}
\subsection{Practical Algorithm}
\label{sec:practical-algorithm}
\vspace{-1mm}

\begin{algorithm}[tb]
\caption{Training procedure for Variational Entropic Optimal Transport (VarEOT)}
\label{alg:vareot}
\begin{algorithmic}
\STATE {\bfseries Input:}
samples from distributions $p_0$ and $p_1$; \\
entropy regularization parameter $\varepsilon > 0$; \\
batch sizes $N_0, N_1$; number of noise samples $K$; \\
potential network $\hat{f}_{\theta} : \mathbb{R}^D \rightarrow \mathbb{R}$; \\
auxiliary network $\hat{\xi}_{\psi} : \mathbb{R}^D \rightarrow \mathbb{R}$. 
\STATE {\bfseries Output:} trained potential $\hat{f}_{\theta^\ast}$.
\vspace{1mm}

\FOR{each training iteration}
    \STATE Sample mini-batches $\{x_i^0\}_{i=1}^{N_0} \sim p_0$,\; $\{x_j^1\}_{j=1}^{N_1} \sim p_1$;
    \STATE Sample i.i.d.\ noise variables $z_{i,k} \sim \mathcal{N}(0,I)$ for $i=1,\dots,N_0$, $k=1,\dots,K$;
    \STATE Compute the empirical loss $\hat{\mathcal{L}}$ according to eq. \eqref{eq:empirical-loss};
    \STATE Update $\psi,\theta$ by using the gradients $\nabla_{\psi}\hat{\mathcal{L}},\nabla_{\theta}\hat{\mathcal{L}}$;
\ENDFOR
\end{algorithmic}
\end{algorithm}

\textbf{Training.} The variational formulation derived in \cref{prop:euclidean-dual} leads to a
fully tractable and simulation-free training optimization objective, given in \eqref{eq:final-loss}. In practice, we parametrize both the potential $f : \mathbb{R}^D \to \mathbb{R}$ and the auxiliary variational
function $\xi : \mathbb{R}^D \to \mathbb{R}$ by neural networks, denoted by
$f_{\theta}$ and $\xi_{\psi}$, respectively

The resulting training procedure consists of maximizing an empirical estimate
of the variational dual loss with respect to both parameter sets $\theta$ and
$\psi$. Crucially, both networks $f_\theta$ and $\xi_\psi$ are trained jointly in a single optimization loop, with no alternating or adversarial steps required. Given mini-batches of samples
$\{x_i^0\}_{i=1}^{N_0} \sim p_0$ from the source distribution and
$\{x_j^1\}_{j=1}^{N_1} \sim p_1$ from the target distribution, the expectations
in eq. \eqref{eq:final-loss} are approximated using Monte Carlo sampling with
i.i.d.\ Gaussian noise variables
$z_{i,k} \sim \mathcal{N}(0,I)$, yielding the empirical loss, up to an \sunset{additive constant} in \eqref{eq:final-loss}:
\begin{equation}
\label{eq:empirical-loss}
\begin{aligned}
\widehat{\mathcal{L}}(f_{\theta}, \xi_{\psi}) 
\defeq 
\frac{1}{N_1} \sum_{j=1}^{N_1}{f_{\theta}(x_j^1)}
-
\varepsilon\frac{1}{N_0} \sum_{i=1}^{N_0}
\xi_{\psi}(x_i^0)
- \\
\frac{\varepsilon}{N_0K}
\sum_{i,k=1}^{N_0,K}
\exp\!\left(
\frac{f_{\theta}(x_i^0 + \sqrt{\varepsilon}\, z_{i,k})}{\varepsilon} - \xi_{\psi}(x_i^0)
\right).
\end{aligned}
\end{equation}
Importantly, \textbf{this approximation \textit{does
not require sampling from the model distribution itself}}, in contrast to
energy-based approaches such as EgNOT. The complete training procedure is
summarized in \cref{alg:vareot}. To mitigate numerical instability arising from the exponential terms in \eqref{eq:empirical-loss}, we clip their values prior to gradient computation and additionally apply gradient clipping during training; see \underline{Appendix~\ref{sec-exp-details}} for details.

\textbf{Inference.} Following the hint in equation \eqref{eq:excess-kl}, after the training we only use the potential function $\hat{f}_{\theta}$ for inference, whereas the function $\hat{\xi}_{\psi}$ is not used.
The corresponding entropic optimal transport plan is implicitly defined via the conditional distribution \eqref{eq:cond_pi_star}. To generate samples from this conditional distribution, we employ Langevin dynamics targeting the unnormalized density
\[
\pi(x_1 \mid x_0) \propto
\exp\!\left(\frac{
\hat{f}_{\theta}(x_1)
-
\frac{1}{2}\|x_1 - x_0\|^2}{\varepsilon}
\right).
\]
The resulting sampling procedure is detailed in
\cref{alg:vareot-sampling}. In all experiments, we find that $10^1$-$10^3$ Langevin steps with an appropriate step size suffice for high-quality samples.

\begin{algorithm}[tb]
\caption{Langevin sampling from the VarEOT conditional distribution}
\label{alg:vareot-sampling}
\begin{algorithmic}
\STATE {\bfseries Input:}
source sample $x_0 \sim p_0$; \\
trained potential network $\hat{f}_{\theta}$; \\
entropy regularization $\varepsilon > 0$; \\
number of Langevin steps $S$; step size $\eta > 0$.%

\STATE {\bfseries Output:} sample $x_1 \sim \pi(\cdot \mid x_0)$.

\vspace{1mm}
\STATE Initialize $x_1^{(0)} = x_0$;
\FOR{$s = 1$ {\bfseries to} $S$}
    \STATE Sample $z^{(s)} \sim \mathcal{N}(0, I)$;
    \STATE $h = \frac{1}{\varepsilon}\left(
    \nabla_{x_1} \hat{f}_{\theta}(x_1^{(s-1)})
    -
    \bigl(x_1^{(s-1)} - x_0\bigr)
    \right)$;
    \STATE Update:
    \vspace*{-4mm}\\
    \[
    x_1^{(s)} \leftarrow
    x_1^{(s-1)}
    +
    \eta h
    +
    \sqrt{2\eta}\, z^{(s)};
    \]
    \vspace*{-5mm}\\
\ENDFOR
\STATE {\bfseries Return} $x_1^{(S)}$
\end{algorithmic}
\end{algorithm}

\begin{figure*}[!t]
\begin{subfigure}[b]{0.245\linewidth}
\centering
\includegraphics[width=0.995\linewidth]{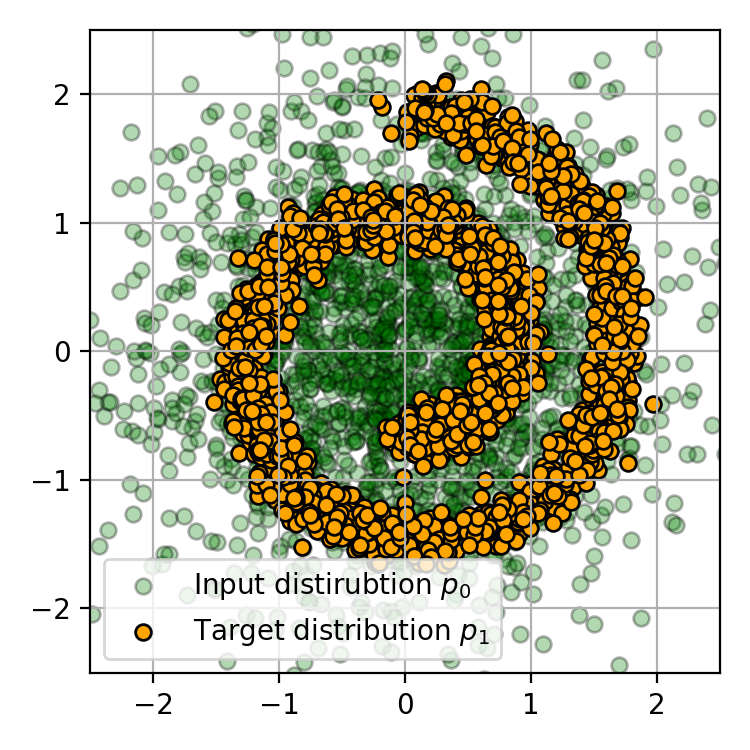}
\caption{\centering ${x_0\sim p_0
}$, ${x_1 \sim p_1}.$}
\vspace{-1mm}
\end{subfigure}
\vspace{-1mm}\hfill\begin{subfigure}[b]{0.245\linewidth}
\centering
\includegraphics[width=0.995\linewidth]{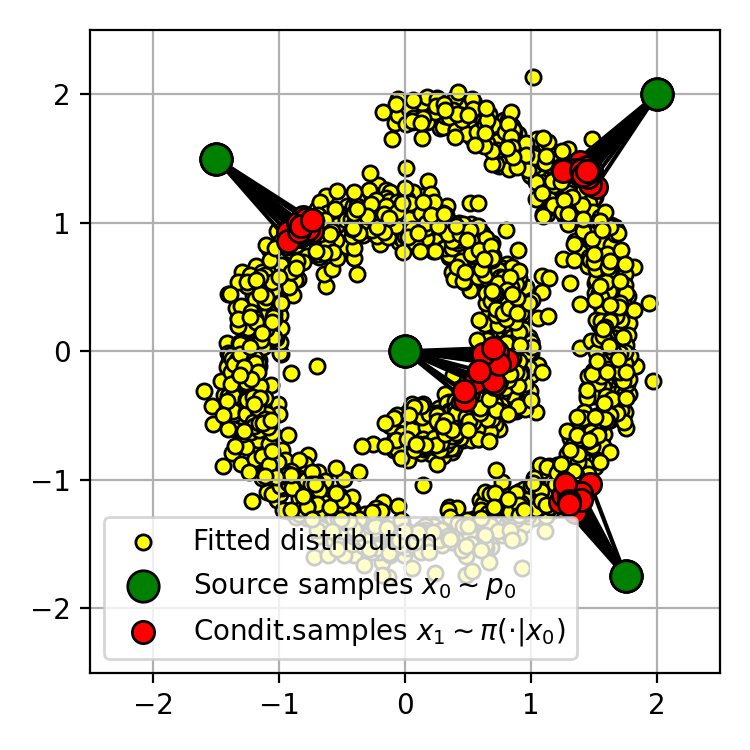}
\caption{\centering $\varepsilon=0.01$.}
\vspace{-1mm}
\end{subfigure}
\hfill\begin{subfigure}[b]{0.245\linewidth}
\centering
\includegraphics[width=0.995\linewidth]{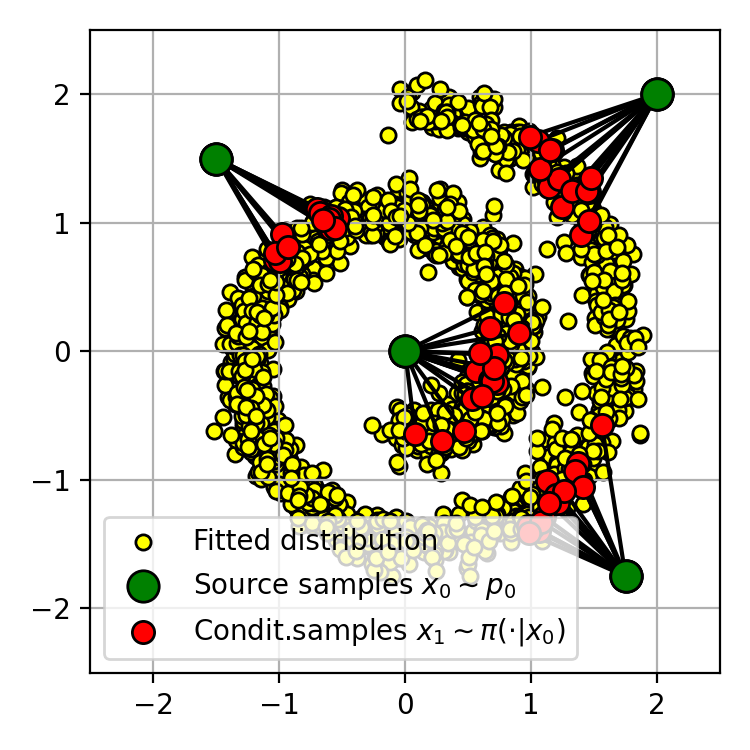}
\caption{\centering $\varepsilon=0.1$.}
\vspace{-1mm}
\end{subfigure}
\hfill\begin{subfigure}[b]{0.245\linewidth}
\centering
\includegraphics[width=0.995\linewidth]{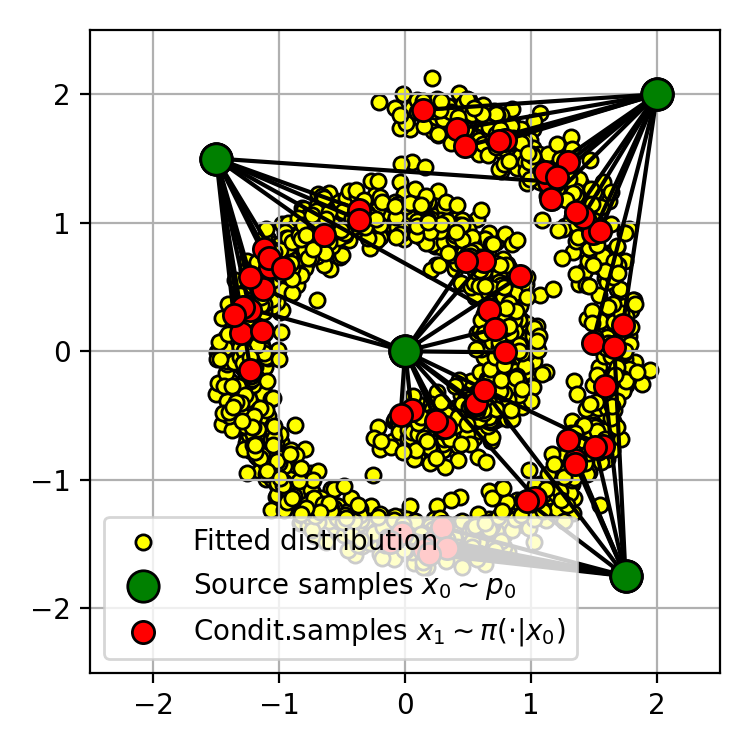}
\caption{\centering $\varepsilon=1.0$.}
\vspace{-1mm}
\end{subfigure}
\vspace{+1mm}
\captionsetup{font=small}
\caption{\centering Optimal plan learned with VarEOT \textbf{(ours)} in \textit{Gaussian} $\!\rightarrow\!$ \textit{Swiss roll} example.}
\label{fig:swiss-roll}
\vspace{-3mm}
\end{figure*}

To highlight the practical advantages of VarEOT, we compare it with two
representative dual-solver approaches: EgNOT~\citep{mokrov2024energyguided} and
LightSB~\citep{korotin2024light} (cf. \wasyparagraph\ref{sec:existing-solvers}). Table~\ref{tab:method-comparison} summarizes
their properties in terms of simulation-free training and structural
constraints on the transport plan.

\begin{table}[h!]
\centering
\caption{\centering Comparison of dual solvers for entropic OT. ``Simulation-free'' indicates
whether the method avoids sampling from the model during training.
``Not Restricted parameterization'' indicates whether the conditional plan is not restricted to a predefined family (e.g., Gaussian mixtures).}
\label{tab:method-comparison}
\vspace{1mm}
\begin{tabular}{l c c}
\toprule
Method & \shortstack{Simulation-free\\training} & \shortstack{Not restricted\\parameterization} \\
\midrule
EgNOT     & \textcolor{red}{$\times$} & \textcolor{green!60!black}{$\checkmark$} \\
LightSB   & \textcolor{green!60!black}{$\checkmark$} & \textcolor{red}{$\times$} \\
VarEOT \textbf{(ours)} & \textcolor{green!60!black}{$\checkmark$} & \textcolor{green!60!black}{$\checkmark$} \\
\bottomrule
\end{tabular}
\vspace{-2mm}
\end{table}

VarEOT combines the best of both worlds: it is
\textit{simulation-free training}, unlike EgNOT, and it imposes no
\textit{restricted parameterization} on the conditional plan, unlike LightSB.

\vspace{-2mm}
\subsection{Finite Sample Learning Guarantees}\label{sec:learning-guarantees}
\vspace{-1mm}

In this subsection we quantify the discrepancy between the transport plan recovered by VarEOT and the ground truth EOT solution. Our method does not have access to $(p_0,p_1)$ explicitly; instead it works with finite samples (see our computational setup \S\ref{sec:practical-setup}) and with
potentials restricted to parametric classes. This naturally introduces several sources of error in practice: {\it finite-sample error} (we only observe i.i.d. datasets of sizes $N$ and $M$ from $p_0$ and $p_1$); {\it function class restriction} (we optimize over a class of neural networks rather than over all continuous potentials); {\it optimization error} (stochastic training, finite-time optimization).
Our rigorous theoretical analysis below focuses on the first two items. The third item depends on the chosen optimizer and sampling scheme and is treated separately in practice.

Throughout, the function classes $\cF$ and $\Xi$ are understood as classes of feedforward neural networks with biases, prescribed depth and width, uniformly bounded layer norms, and clipped outputs. Here, clipped outputs mean that the final network output is truncated to lie in a fixed bounded interval. We additionally assume that the hidden activation $\sigma$ is $L_\sigma$-Lipschitz, non-polynomial and satisfies $\sigma(0)=0$. Thus,
\vspace{-0mm}
\begin{equation}\label{eq:classesFXI}
\cF \;:=\; \mathsf{NN}_\theta^{\sigma}(D,1), \quad
\Xi \;:=\; \mathsf{NN}_\eta^{\sigma}(D,1),
\end{equation}
We begin with a decomposition of the excess $\mathrm{KL}$ discrepancy between the ground-truth EOT optimizer $\pi^\ast$ and the learned plan $\pi^{\hat f}$ (cf. \eqref{eq:excess-kl}), where
\begin{align*}
(\hat f,\hat \xi)
=
\argmax_{f\in\cF,\ \xi\in\Xi}\widehat{\mathcal L}(f,\xi).
\end{align*}
Here $\xi$ enters only through the training objective, whereas the transport plan used at inference time is determined exclusively by the learned potential $\hat f$; accordingly, the final discrepancy is measured in terms of $\pi^{\hat f}$, cf. \S\ref{sec:practical-algorithm}.
\begin{proposition}\label{prop:est-approx-decay}
    The following bound holds:
    \begin{equation}
    \begin{aligned}
    \hspace*{-3mm}\varepsilon\mathbb{E}\!\Big[
    \mathrm{KL}\big(\pi^\ast \,\Vert\, \pi^{\hat f}\big)
    \!\Big]
    \!\le &
    \underbrace{
    2\mathbb{E}\!\!\!\!\sup_{f\in\cF,\ \xi\in\Xi}\!\left|{\mathcal L}(f,\xi)\!-\!\widehat{{\mathcal L}}(f,\xi)\right|
    }_{\text{Estimation error}}
    +\\[0.4em]
    &\underbrace{
    \mathcal{L}^*-\sup_{f\in\cF, \xi\in\Xi}\mathcal{L}(f,\xi)
    }_{\text{Approximation error}},
    \end{aligned}
    \end{equation}
    where the expectations are taken w.r.t. the random realization of the datasets.
\end{proposition}

Proposition~\ref{prop:est-approx-decay} separates the excess error into a statistical term (estimation error) and a representational term (approximation error). The first measures the uniform discrepancy between the empirical and population objectives over the prescribed neural-network classes, while the second quantifies the gap from restricting the optimization to these classes.

The next theorem controls the statistical term. It shows that, for fixed network classes, the estimation error decays at the usual sublinear rate as the number of samples increases. Importantly, decay rates do not depend on the dimension $D$.

\begin{theorem}[Bound on estimation error]\label{thm:estim}
Assume that the classes $\cF$ and $\Xi$ are defined by \eqref{eq:classesFXI}, and that $p_0$ and $p_1$ have compact supports. Then
\begin{equation}
    \begin{aligned}
    \mathbb{E}\sup_{f\in\cF,\ \xi\in\Xi}\left|{\mathcal L}(f,\xi)-\widehat{{\mathcal L}}(f,\xi)\right|\le
    \\
    \le O\left(N^{-\frac12}\right)+O\left(M^{-\frac12}\right) + O\left(K^{-\frac12}\right)
    \end{aligned}
    \end{equation}
where the hidden constants depend on $\varepsilon$, the dimension $D$, the complexity parameters of the classes $\cF$ and $\Xi$ (in particular, the layer-norm bounds, clipping level, width, and depth), and the support diameters of $p_0$ and $p_1$. Expectation is taken w.r.t. random realizations of the dataset.
\end{theorem}

\begin{table*}[!t]
\centering
\vspace{5mm}
\resizebox{\textwidth}{!}{%
\begin{tabular}{ |c|c|c|c|c|c| }
\hline
\textbf{Setup} & \textbf{Solver type} & \backslashbox{\textbf{Solver}}{\textbf{DIM}} & \textbf{50} & \textbf{100} & \textbf{1000} \\
\hline
 Discrete EOT & Sinkhorn & \citep{cuturi2013sinkhorn} [1 GPU V100]* & $2.34$ ($90$ s) & $2.24$ ($2.5$ m) & $1.864$ ($9$ m) \\
\hline
 Continuous EOT & Langevin-based & \citep{mokrov2024energyguided} [1 GPU V100]* & $2.39 \pm 0.06$ ($19$ m) & $2.32 \pm 0.15$ ($19$ m) & $1.46 \pm 0.20$ ($15$ m) \\
\hline
 Continuous EOT & Minimax & \citep{gushchin2023entropic} [1 GPU V100]* & $2.44 \pm 0.13$ ($43$ m) & $2.24 \pm 0.13$ ($45$ m) & $1.32 \pm 0.06$ ($71$ m) \\
\hline
 Continuous EOT & IPF & \citep{vargas2021solving} [1 GPU V100]* & $3.14 \pm 0.27$ ($8$ m) & $2.86 \pm 0.26$ ($8$ m) & $2.05 \pm 0.19$ ($11$ m) \\
\hline
 Continuous EOT & Variational & \textbf{VarEOT (ours)} [GTX 1080] & $\mathbf{2.45 \pm 0.08}$ ($3$ m) & $\mathbf{2.31 \pm 0.15}$ ($3$ m) & $\mathbf{1.47 \pm 0.11}$ ($14.8$ m) \\
\hline
\end{tabular}%
}
\vspace{2mm}
\captionsetup{font=small, justification=centering}
\caption{\color{black} Energy distance (averaged for two setups and 5 random seeds) on the MSCI dataset along with $95\%$-confidence interval ($\pm$ intervals) and average training times (s~--- seconds, m~--- minutes). Results marked with * are taken from \citep{korotin2024light}.}
\label{table-sc-comparison}
\vspace{-3mm}
\end{table*}

We next turn to the approximation term. The following statement asserts that it can be made arbitrarily small by choosing sufficiently expressive neural-network classes.

\begin{theorem}[Vanishing of approximation error]\label{thm:approx} Let $p_0$ and $p_1$ have compact supports, then for any $\delta>0$, there exist classes $\cF$ and $\Xi$ of the form \eqref{eq:classesFXI} such that
\begin{equation}
\mathcal L^\ast-\sup_{f\in\cF,\ \xi\in\Xi}\mathcal L(f,\xi)
<
\delta.
\end{equation}
\end{theorem}

Thus, the bias introduced by the neural-network parameterization can be driven arbitrarily close to zero by enlarging the expressive power of the admissible classes.

\textbf{Summary.}
Proposition~\ref{prop:est-approx-decay}, together with Theorems~\ref{thm:estim} and~\ref{thm:approx}, yields the standard learnability picture for EOT plan recovery. For fixed neural-network classes, the estimation gap vanishes as the sample sizes grow. Meanwhile, by increasing the expressiveness of the classes, the approximation error can be made arbitrarily small. Consequently, with sufficiently many samples and sufficiently expressive neural-network parameterizations, the learned plan $\pi^{\hat f}$ approaches the EOT solution in excess $\mathrm{KL}$ discrepancy.

\paragraph{Relation to prior work.}
Related error decompositions appear in \cite{mokrov2024energyguided,Kolesov2024EnergyGuidedBarycenter}. 
While \cite{mokrov2024energyguided} does not provide a fully detailed breakdown of the individual error terms for EgNOT, \cite{Kolesov2024EnergyGuidedBarycenter} develops a finer-grained statistical analysis but for a different objective (entropic barycenters). 
We follow a similar template, yet provide a detailed analysis for our more elaborate variationally derived VarEOT functional.

\vspace{3mm}
\section{Related Work}
In this section, we review the existing continuous EOT approaches, and briefly compare them with our solver. For clarity, we introduce a taxonomy of the existing approaches below, based on the utilized mathematical framework.

\textbf{Weak dual EOT} (a.k.a. \textit{semi}-dual EOT) methods \citep{mokrov2024energyguided, korotin2024light} optimize the objective eq. \eqref{eq:dual-final}. They are the most similar to our VarEOT in terms of methodology. The detailed discussion of the approaches is given in \wasyparagraph\ref{sec:existing-solvers}. Crucially, our proposed method successfully overcomes the limitations of existing weak dual approaches while retaining their advantages.

\textbf{Dual EOT} \citep{genevay2016stochastic, seguy2018large, daniels2021score} exploit an alternative dual form of the EOT. These methods simultaneously optimize a \textit{pair} of dual potentials $(u, v)$ in a $\max \max$ optimization procedure known as the Sinkhorn algorithm. While bearing certain resemblance to our VarEOT (e.g., simulation-free training), dual EOT approaches have pitfalls underscoring their performance:
\begin{itemize}[leftmargin=*]
    \item Dual potentials $(u, v)$ do not fully recover the (conditional) EOT  $\pi^*(\cdot \vert x_0)$. In particular, an \textit{auxiliary} score-based model for the target distribution $p_1$ is required to enable sampling from $\pi^*(\cdot \vert x_0)$ \citep{daniels2021score}. Otherwise, there are only some heuristics (e.g., barycentric projection) that allow the recovered solution to be cast as a \textit{generative} model capable of data-to-data translation \citep{genevay2016stochastic, seguy2018large}.
    \item Due to optimization peculiarities, dual EOT approaches tend to be unstable under small regularization coefficient $\varepsilon$. This behavior is reported in \citep[\wasyparagraph 5.1]{daniels2021score} and supported by \citep[\wasyparagraph 4.2, \wasyparagraph C.2]{mokrov2024energyguided}.
\end{itemize}
Overall, VarEOT provides a more user-friendly and production-ready framework with less engineering overhead required to adapt the method to applications.

\textbf{Schr\"odinger bridge (SB)} methods \citep{de2021diffusion, vargas2021solving, chen2022likelihood, gushchin2023entropic, shi2023diffusion,tong2024simulation,gushchin2024adversarial,de2024schrodinger,gushchin2024light} cast the EOT in a dynamic manner and recover a solution in the form of a stochastic differential equation with learnable drift, with the endpoints given by source $p_0$ and target $p_1$ distributions. The underlying principles of SB-based approaches are different, e.g., adversarial optimization, gradual iterative markovian/proportional fitting procedure. However, compared to our method, \textbf{majority} of them are not simulation-free at training. The exceptions are \citep{tong2024simulation} (based on mini-batch EOT approximation, which may not be accurate) and \citep{gushchin2024light} (similar to \citep{korotin2024light} uses restricted Gaussian mixture approximation).

\vspace{3mm}
\section{Experimental Illustrations}
\label{sec-experiments}
\vspace{-1mm}

We evaluate VarEOT on both synthetic and real-world data.
\wasyparagraph\ref{sec-exp-toy} presents illustrative two-dimensional experiments
demonstrating the effect of the regularization parameter $\varepsilon$ on the
learned transport plan. \wasyparagraph\ref{sec-exp-new-single-cell} benchmarks VarEOT on high-dimensional biological single-cell data from the MSCI dataset, comparing against EOT/SB solvers. \wasyparagraph\ref{sec-exp-alae} comprehensively evaluates on unpaired image-to-image translation using the standard ALAE
latent space protocol, comparing VarEOT against related entropic transport
baselines across multiple translation tasks and regularization regimes.
\underline{Technical details}, including architecture specifications and hyperparameters, are provided in Appendix~\ref{sec-exp-details}.

\vspace{-2mm}
\subsection{Two-dimensional Examples}
\label{sec-exp-toy}
\vspace{-1mm}

We begin with a standard illustrative two-dimensional experiment that provides intuition into the role of the entropy regularization parameter $\varepsilon$ in shaping the
learned conditional transport plan $\pi_{\theta}(x_1 \mid x_0)$. Specifically,
we consider a setting in which a Gaussian source distribution is transported to
a Swiss Roll target distribution. We apply our method for
$\varepsilon \in \{10^{-2},\,10^{-1},\,10^{0}\}$, and visualize the resulting
transport maps in Fig.~\ref{fig:swiss-roll}.  

For small values of $\varepsilon$, the learned transport is nearly deterministic,
with little variation in sampling endpoints. As $\varepsilon$ increases, the transport becomes progressively more stochastic: the endpoints diversify, and the conditional
distributions $\pi_{\theta}(x_1 \mid x_0)$ spread over a wider region of the
target space.

\begin{figure*}[t]
    \centering
    \includegraphics[width=1.0\textwidth]{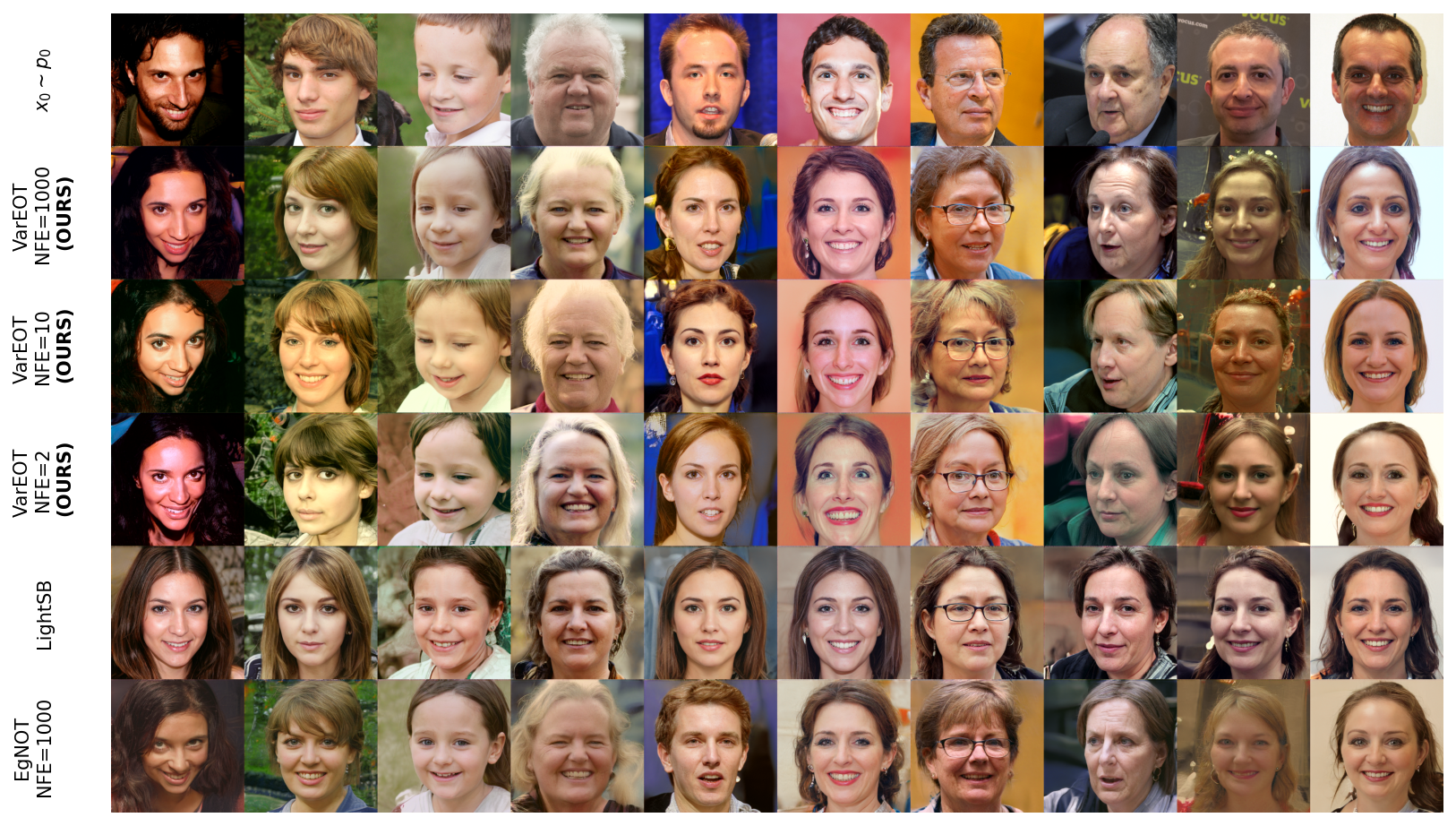}
    \caption{\centering
    Qualitative comparison for \textit{Man} $\rightarrow$ \textit{Woman} translation with $\varepsilon=1.0$.
    From top to bottom: input samples, VarEOT (ours), LightSB, and EgNOT.
    Input images are selected from the test set: we take the first 300 samples and rank them by encoder-decoder reconstruction quality (LPIPS), displaying the top-ranked examples.}
    \label{fig:male2female}
    \vspace{-3mm}
\end{figure*}

\vspace{-2mm}
\subsection{Single Cell Data}
\label{sec-exp-new-single-cell}
\vspace{-1mm}

One of the important applications of EOT is the analysis of biological single-cell data \citep{koshizuka2023neural, bunne2023learning, bunne2022proximal}. We evaluate VarEOT on the high-dimensional \underline{MSCI dataset} (Kaggle competition ``Open Problems~--- Multimodal Single-Cell Integration'') first used in \citep{tong2024simulation}. The dataset consists of single-cell data from four human donors at 4 time points (days $2$, $3$, $4$, and $7$). We solve the EOT problem between distribution pairs at days $2\to4$ and $3\to7$, and evaluate how well the solvers recover the intermediate distributions at days $3$ and $4$ respectively. We consider PCA projections with $\text{DIM}=50, 100, 1000$ components and report energy distance \citep{rizzo2016energy} (ED) in Table~\ref{table-sc-comparison}. For VarEOT, the experimental setup (architecture, number of iterations, $\varepsilon$, and learning rate) follows the EgNOT configuration from \citep{korotin2024light} with $K{=}128$ particles; we refer the reader to \underline{Appendix~D.4} of \citep{korotin2024light} for full details. VarEOT achieves quality comparable to Langevin-based and Minimax solvers for $\text{DIM}=50$ and $100$, while being $5$--$15\times$ faster ($3$ m vs.\ $19$--$45$ m) on less powerful hardware (GTX~1080 vs.\ V100). For $\text{DIM}=1000$, VarEOT matches the quality of the Langevin-based solver ($1.47$ vs.\ $1.46$) with comparable training time ($14.8$ m vs.\ $15$ m). The \underline{details} of preprocessing, hyperparameters and baselines are in Appendix~\ref{sec-exp-details}.

\subsection{Unpaired Image-to-Image Translation}
\label{sec-exp-alae}

Unpaired image-to-image translation is a standard benchmark in the
entropic optimal transport (EOT) and Schrödinger Bridge (SB) literature
\citep{zhu2017unpaired,daniels2021score,chen2021stochastic}, as it
naturally combines high-dimensional data and the need for stochastic transport plans. In particular, translation in the latent
space of pretrained autoencoders has become a popular and convenient
setup for comparing entropic transport methods \cite{korotin2024light, mokrov2024energyguided, kornilov2024optimal}.

We follow the widely adopted ALAE protocol \citep{korotin2024light, theodoropoulos2024feedback, gushchin2024light, kornilov2024optimal, gazdieva2024light, han2025variational} based on ALAE autoencoder \citep{pidhorskyi2020adversarial}. Specifically, we use a pretrained ALAE
autoencoder trained on the full $1024 \times 1024$ FFHQ dataset
\citep{karras2019style}, which contains approximately $70$K human face
images. The first $60$K images are used for training and are split into
(\textit{male}, \textit{female}) and (\textit{child}, \textit{adult}) subsets. Using the fixed ALAE encoder,
we extract $512$-dimensional latent representations
$\{z_0^n=\mathrm{Enc}(x_0^n) \}_{n=1}^N$ and $\{z_1^m=\mathrm{Enc}(x_1^m)\}_{m=1}^M$ corresponding to the source
and target domains. We consider $4$ setups: male to female (M$\to$F), female to male (F$\to$M), adult to child (A$\to$C), and child to adult (C$\to$A).

\vspace{-1mm}
\textbf{Training.}
Given unpaired samples from the two latent distributions, we learn a
latent entropic optimal transport plan
$\pi_\theta(z_1 \mid z_0)$ using our variational formulation. The model
is trained entirely in latent space and does not require paired data or
image-level supervision.

\vspace{-1mm}
\textbf{Inference.}
To translate a previously unseen \textit{image} $x_0^{\text{new}}$
(from the remaining $10$K test images), we
\textbf{(i)} encode it as $z_0^{\text{new}} = \mathrm{Enc}(x_0^{\text{new}})$,
\textbf{(ii)} sample $z_1 \sim \pi_\theta(z_1 \mid z_0^{\text{new}})$ using
Langevin dynamics, and \textbf{(iii)} decode 
$x_1 = \mathrm{Dec}(z_1)$.

\textbf{Evaluation Protocol.}
We evaluate the quality of generated samples using two metrics: Fréchet Inception Distance (FID) \citep{heusel2017gans} and Learned Perceptual Image Patch Similarity (LPIPS) \citep{zhang2018unreasonable}. FID is computed between the set of test translated images $\{x_1^n = \mathrm{Dec}(z_1^n)\}_{n=1}^{N}$, where $z_1^n \sim \pi_\theta(z_1 \mid z_0^n)$, and the set of ALAE-reconstructed target images $\{\tilde{x}_1^m = \mathrm{Dec}(\mathrm{Enc}(x_1^m))\}_{m=1}^{M}$. LPIPS is measured between the source $x_0$ and the translated $x_1 = \mathrm{Dec}(z_1)$, where $z_1 \sim \pi_\theta(z_1 \mid \mathrm{Enc}(x_0))$. It measures the input-output similarity.

\begin{table*}[!t]
\scriptsize
\setlength{\tabcolsep}{5pt}
\begin{tabular}{cc|
cccccc | cccccc}
\toprule
& &
\multicolumn{6}{c}{FID $\downarrow$} &
\multicolumn{6}{c}{LPIPS $\downarrow$} \\
\cmidrule(lr){3-8} \cmidrule(lr){9-14}
$\varepsilon$ & Task
& LightSB
& \multicolumn{2}{c}{EgNOT}
& \multicolumn{3}{c}{VarEOT \textbf{(ours)}}
& LightSB
& \multicolumn{2}{c}{EgNOT}
& \multicolumn{3}{c}{VarEOT \textbf{(ours)}} \\
\cmidrule(lr){4-5}\cmidrule(lr){6-8}\cmidrule(lr){10-11} \cmidrule(lr){12-14}
& &
&
NFE=10 & NFE=1000 &
NFE=2 & NFE=10 & NFE=1000
& &
NFE=10 & NFE=1000
& NFE=2 & NFE=10 & NFE=1000 \\
\midrule
\multirow{4}{*}{$0.1$}
& M$\to$F & 8.77 & 13.4 & \textbf{7.63} & 181.44 & 11.70 & 12.74 & 0.582 & 0.6 & 0.585 & 0.872 & 0.580 & \textbf{0.570} \\
& F$\to$M & 10.87 & 14.9 & \textbf{6.33} & 248.76 & 17.00 & 13.98 & 0.597 & 0.603 & 0.598 & 0.833 & 0.586 & \textbf{0.578} \\
& A$\to$C & 15.42 & 40.28 & \textbf{10.52} & 247.81 & 50.35 & 17.33 & 0.577 & 0.586 & 0.588 & 0.839 & \textbf{0.545} & 0.577 \\
& C$\to$A & 13.46 & 18.47 & \textbf{10.4} & 253.59 & 21.48 & 20.13 & 0.594 & 0.602 & 0.608 & 0.865 & 0.591 & \textbf{0.590} \\
\midrule
\multirow{4}{*}{$0.5$}
& M$\to$F & 19.09 & 35.4 & 14.07 & \textbf{9.04} & 9.50 & 9.77 & 0.611 & 0.718 & 0.594 & 0.619 & 0.600 & \textbf{0.598} \\
& F$\to$M & 25.81 & 51.73 & 11.66 & 20.01 & 10.02 & \textbf{9.94} & 0.628 & 0.695 & 0.613 & \textbf{0.609} & 0.619 & 0.616 \\
& A$\to$C & 22.30 & 74.25 & \textbf{18} & 78.85 & 26.38 & 25.99 & 0.608 & 0.702 & 0.603 & \textbf{0.593} & 0.617 & 0.613 \\
& C$\to$A & 22.70 & 40.38 & 16.95 & 17.83 & \textbf{16.55} & 16.74 & \textbf{0.614} & 0.7 & 0.629 & 0.632 & 0.616 & \textbf{0.614} \\
\midrule
\multirow{4}{*}{$1.0$}
& M$\to$F & 22.63 & 39.41 & 17.58 & 10.38 & \textbf{9.59} & 9.71 & 0.637 & 0.749 & 0.634 & 0.633 & 0.613 & \textbf{0.610} \\
& F$\to$M & 20.97 & 66.64 & 26.21 & 15.21 & 10.52 & \textbf{10.28} & 0.649 & 0.72 & 0.634 & 0.654 & 0.633 & \textbf{0.630} \\
& A$\to$C & 24.43 & 60.57 & 28.45 & 53.39 & \textbf{16.60} & 16.83 & 0.634 & 0.724 & 0.623 & 0.700 & 0.610 & \textbf{0.608} \\
& C$\to$A & 24.87 & 51.48 & 23.68 & \textbf{15.73} & 16.55 & 16.70 & 0.637 & 0.733 & 0.661 & 0.638 & 0.623 & \textbf{0.619} \\
\midrule
\multirow{4}{*}{$10.0$}
& M$\to$F
& 31.85 & 37.61 & 26.68 & \textbf{21.84} & 21.97 & 22.38 & 0.680 & 0.755 & 0.735 & 0.625 & 0.621 & \textbf{0.620} \\
& F$\to$M & 34.47 & 63.73 & 36 & \textbf{25.55} & 28.84 & 29.24 & 0.693 & 0.758 & 0.734 & 0.637 & 0.632 & \textbf{0.630} \\
& A$\to$C & \textbf{31.85} & 92.6 & 64.85 & 49.84 & 50.72 & 51.20 & 0.680 & 0.767 & 0.763 & 0.619 & \textbf{0.615} & \textbf{0.615} \\
& C$\to$A & 35.35 & 51.74 & 35.53 & \textbf{31.01} & 32.84 & 34.07 & 0.683 & 0.748 & 0.741 & 0.639 & 0.636 & \textbf{0.633} \\
\bottomrule
\end{tabular}
\vspace{2mm}
\caption{\centering
Quantitative comparison of unpaired image-to-image translation methods
in the ALAE latent space. We report FID and LPIPS (lower is better) for
four translation tasks (M$\to$F, F$\to$M, A$\to$C, C$\to$A) across different values of
the entropic regularization parameter $\varepsilon$.
VarEOT results are shown for different numbers of Langevin inference
steps (NFE $=2, 10, 1000$).
For each row, the best-performing method is highlighted in
\textbf{bold}.
}
\vspace{0mm}
\label{tab:fid-lpips}
\end{table*}

\textbf{Results.}
We comprehensively evaluate VarEOT across all four translation setups:
M$\to$F, F$\to$M, A$\to$C, and C$\to$A. Figure~\ref{fig:male2female}
shows qualitative results for the M$\to$F setup with $\varepsilon = 1.0$.
We report translations obtained using different numbers of Langevin steps
at inference time (NFE $= 2, 10, 1000$), demonstrating that VarEOT
produces visually plausible and stable results even with a very small
number of sampling steps. \underline{Additional results} for alternative
translation directions (F$\to$M, A$\to$C, C$\to$A) are provided in
Appendix~\ref{sec-exp-additional}
(Figures~\ref{fig:W2M},~\ref{fig:O2Y}, and~\ref{fig:Y2O}), confirming
consistent behavior across all setups. Quantitative evaluation using FID and LPIPS
metrics is summarized in Table~\ref{tab:fid-lpips},
covering a wide range of regularization parameters
$\varepsilon \in \{0.1, 0.5, 1.0, 10.0\}$ and all four translation
directions. We compare VarEOT against closely related entropic transport
baselines: LightSB and EgNOT, showing competitive or improved performance
across all considered regimes.

A qualitative analysis of the effect of the
regularization parameter $\varepsilon$ on sample diversity is provided in Figure~\ref{fig:eps-dependence}, illustrating how varying $\varepsilon$
controls the trade-off between transport cost and output stochasticity.
Finally, Figure~\ref{fig:fid-heatmaps} provides an in-depth study of FID
performance as a function of the Langevin step size and the number of
inference steps.

We additionally conduct a wall-clock time analysis as a function of the number of particles $K$. Table~\ref{tab:wall-clock} shows that VarEOT training time scales with $K$: larger $K$ leads to longer training, but also to better FID. Importantly, VarEOT achieves superior FID compared to both LightSB and EgNOT across all considered values of $K$, while requiring significantly less training time than EgNOT (e.g., $131$s vs.\ $1692$s for better quality). We attribute this to our simulation-free training objective, which avoids costly trajectory simulation required by EgNOT.

\vspace{3mm}
\section{Discussion}
\textbf{Potential impact.}
By introducing an exact variational reformulation of the weak dual EOT
objective, VarEOT makes a step toward more efficient entropic transport
algorithms that avoid the key limitations of existing methods, such as
simulation-based training, adversarial objectives, and restrictive
parametric assumptions.

\textbf{Limitations.}
While VarEOT enables simulation-free training, several limitations remain. First, sampling at inference time via Langevin dynamics is still required, making generation quality dependent on step size and step count. Second, the exponential terms in the objective may cause numerical instability during training. Third, the quality of optimization is inherently tied to the expressiveness of the auxiliary network $\xi$, and the choice of its architecture remains a practical question; we note that in our experiments a standard MLP for $\xi$ proved sufficient. Finally, the current framework is specialized to the quadratic cost (Gaussian kernel); extending to general costs is a natural direction for future work.

\section*{Acknowledgements}
The work was supported by the grant for research centers in the field of
AI provided by the Ministry of Economic Development of the Russian
Federation in accordance with the agreement 000000C313925P4F0002 and the
agreement №139-10-2025-033.

\section*{Impact Statement}
This paper presents work whose goal is to advance the field of Machine
Learning. There are many potential societal consequences of our work, none
which we feel must be specifically highlighted here.
\clearpage

\nocite{langley00}

\bibliography{example_paper}
\bibliographystyle{icml2026}

\newpage
\appendix
\onecolumn

\section{Proofs}\label{app:Proofs}

 Throughout the appendix, we will use the shorthand $Z_f := Z(f,\cdot)$.

\begin{definition}\label{def:KLdiv}
Let $\mu$ and $\nu$ be non-negative measures on $\mathbb{R}^d$, and $\mu$ absolutely continuous w.r.t. $\nu$. The Kullback-Leibler divergence is defined by 
$$
\mathrm{KL}(\mu\| \nu) := \int \left[\frac{d\mu}{d\nu}\log\left(\frac{d\mu}{d\nu}\right) - \frac{d\mu}{d\nu}+1\right]\; d\nu.
$$
    
\end{definition}

\begin{definition}[Empirical Rademacher complexity]
Let $S=(x_1,\dots,x_N)\in X^N$ be a fixed sample, and let $\mathcal H$ be a class of
measurable functions $h:X\to\mathbb R$. Its empirical Rademacher complexity on $S$
is defined by
\[
\widehat{\mathfrak R}_S(\mathcal H)
:=
\frac1N\,\mathbb E_\sigma\left[\sup_{h\in\mathcal H}\sum_{i=1}^N \sigma_i h(x_i)\right],
\]
where $\sigma_1,\dots,\sigma_N$ are independent Rademacher random variables.
\end{definition}

\begin{definition}[Rademacher complexity]
Let $p$ be a probability measure on $X\subseteq\R^D$, and let $\mathcal H$ be a class of measurable
functions $h:X\to\mathbb R$. Its Rademacher complexity is
\[
\mathfrak R_N(\mathcal H,p)
:=
\mathbb E_{S\sim p^N}\,\widehat{\mathfrak R}_S(\mathcal H)
=
\frac1N\,\mathbb E_{x,\sigma}\left[\sup_{h\in\mathcal H}\sum_{i=1}^N \sigma_i h(x_i)\right],
\]
where $x_1,\dots,x_N$ are i.i.d.\ samples from $p$.
\end{definition}

The following symmetrization bound is standard; see, e.g.,
\citep[Lemma~26.2]{shalev2014understanding}.

\begin{lemma}[Representativeness estimation]\label{lem:sym}
Let $\mathcal H$ be a class of measurable functions $h:X\to\mathbb R$ such that
$\mathcal H\subset L^1(p)$. Then for i.i.d.\ samples $x_1,\dots,x_N\sim p$,
\[
\mathbb E\sup_{h\in\mathcal H}\left|
\mathbb E_p[h]-\frac1N\sum_{i=1}^N h(x_i)
\right|
\le 2\,\mathfrak R_N(\mathcal H,p).
\]
\end{lemma}

We will use the vector contraction inequality of \citep[Corollary~4]{maurer2016vector}.

\begin{lemma}[Vector contraction]\label{lem:vector-contraction}
Let $S=(x_1,\dots,x_N)$ be a fixed sample, let $\mathcal F$ be a class of functions
$f=(f_1,\dots,f_K):X\to\mathbb R^K$, and let $h_i:\mathbb R^K\to\mathbb R$ be
$L$-Lipschitz with respect to the Euclidean norm. Then
\[
\mathbb E_\sigma\left[
\sup_{f\in\mathcal F}\sum_{i=1}^N \sigma_i\, h_i(f(x_i))
\right]
\le
\sqrt{2}\,L\,
\mathbb E_{\sigma_{ik}}\left[
\sup_{f\in\mathcal F}\sum_{i=1}^N\sum_{k=1}^K \sigma_{ik} f_k(x_i)
\right],
\]
where $\{\sigma_{ik}\}_{1\le i\le N,\ 1\le k\le K}$ is an independent doubly indexed
Rademacher family.
\end{lemma}

\begin{lemma}
\label{lem:2d-contract}
Let $\mathcal U,\mathcal V$ be classes of measurable functions from $X$ to $\mathbb R$, and
let a function $\psi$ satisfy 
\[
|\psi(u,v)-\psi(u',v')|
\le L_u |u-u'| + L_v |v-v'|
\]
for any pairs $(u,v),(u',v')$ from the range of values $\mathcal U,\mathcal V$.
Then
\[
\mathfrak R_N\bigl(
\{\psi(u(\cdot),v(\cdot)):\ u\in\mathcal U,\ v\in\mathcal V\},\,p
\bigr)
\le
2L_u\,\mathfrak R_N(\mathcal U,p)
+
2L_v\,\mathfrak R_N(\mathcal V,p).
\]
\end{lemma}

\begin{proof}
Denote
\[
\mathcal G
:=
\{\psi(u(\cdot),v(\cdot)):\ u\in\mathcal U,\ v\in\mathcal V\}.
\]
Fix a sample $S=(x_1,\dots,x_N)$ and define the vector-valued class
\[
\mathcal F
:=
\{x\mapsto (L_u u(x),\,L_v v(x)):\ u\in\mathcal U,\ v\in\mathcal V\}.
\]
Also define, %
\[
h(a,b):=\psi(a/L_u,b/L_v)
\]
For all feasible pairs
$(a,b),(a',b')$, we have
\[
|h_i(a,b)-h_i(a',b')|
=
\left|
\psi(a/L_u,b/L_v)-\psi(a'/L_u,b'/L_v)
\right|
\]
\[
\le
L_u\left|\frac{a-a'}{L_u}\right|
+
L_v\left|\frac{b-b'}{L_v}\right|
=
|a-a'|+|b-b'|
\le
\sqrt{2}\,\|(a,b)-(a',b')\|_2.
\]
Thus each $h_i$ is $\sqrt{2}$-Lipschitz with respect to the Euclidean norm.

Now apply Lemma~\ref{lem:vector-contraction} with $K=2$ and $L=\sqrt{2}$:
\[
\widehat{\mathfrak R}_S(\mathcal G)
=
\frac1N\,
\mathbb E_\sigma\left[
\sup_{u\in\mathcal U,\ v\in\mathcal V}
\sum_{i=1}^N \sigma_i \psi(u(x_i),v(x_i))
\right]
\le
\frac{2}{N}\,
\mathbb E_{\sigma_{i1},\sigma_{i2}}
\left[
\sup_{u\in\mathcal U,\ v\in\mathcal V}
\sum_{i=1}^N
\bigl(
\sigma_{i1}L_u u(x_i)+\sigma_{i2}L_v v(x_i)
\bigr)
\right].
\]
Using $\sup(A+B)\le \sup A+\sup B$, we obtain
\[
\widehat{\mathfrak R}_S(\mathcal G)
\le
\frac{2L_u}{N}\,
\mathbb E_{\sigma_{i1}}
\left[
\sup_{u\in\mathcal U}\sum_{i=1}^N \sigma_{i1}u(x_i)
\right]
+
\frac{2L_v}{N}\,
\mathbb E_{\sigma_{i2}}
\left[
\sup_{v\in\mathcal V}\sum_{i=1}^N \sigma_{i2}v(x_i)
\right].
\]
Since $\{\sigma_{i1}\}_{i=1}^N$ and $\{\sigma_{i2}\}_{i=1}^N$ are again independent
Rademacher families, this is exactly
\[
\widehat{\mathfrak R}_S(\mathcal G)
\le
2L_u\,\widehat{\mathfrak R}_S(\mathcal U)
+
2L_v\,\widehat{\mathfrak R}_S(\mathcal V).
\]
Taking expectation over $S\sim p^N$ yields
\[
\mathfrak R_N(\mathcal G,p)
\le
2L_u\,\mathfrak R_N(\mathcal U,p)
+
2L_v\,\mathfrak R_N(\mathcal V,p).
\]
\end{proof}

\subsection{Proof of Proposition~\ref{lem:variational-Z}}

\begin{proof}
Make the change of variables $x_1=x_0+\sqrt{\varepsilon}\,z$, so that $dx_1=\varepsilon^{D/2}dz$ and hence
\[
Z_f(x_0)
=(2\pi\varepsilon)^{D/2}\,\E_{z\sim\mathcal{N}(0,I)}\!\left[\exp\!\left(\frac{f(x_0+\sqrt{\varepsilon}z)}{\varepsilon}\right)\right]
=\E_{z\sim \mathcal{N}(0,I)}\!\left[\exp\!\left(A(z)\right)\right],
\]
where
\[
A(z):=\frac{f(x_0+\sqrt{\varepsilon}z)}{\varepsilon}+\underbrace{\frac{D}{2}\log(2\pi\varepsilon)}_{\defeq C}.
\]
Then for any $\xi(x_0)\in\R$,
\[
\log Z_f(x_0)=\log \E_{z\sim \mathcal{N}(0,I)}[e^{A}]=(\xi(x_0) + C)+\log \E[e^{A-(\xi(x_0) + C)}]
\le (\xi(x_0) + C)+\bigl(\E[e^{A-(\xi(x_0) + C)}]-1\bigr),
\]
using pointwise $\log u\le u-1$ for all $u>0$.
Equality holds iff $\E[e^{A- (\xi + C)}]=1$, i.e. when $\xi + C=\log \E[e^{A}]=\log Z_f(x_0)$.
\end{proof}

\subsection{Proof of Theorem~\ref{prop:euclidean-dual}}
\label{app:Proofs:prop_euclidean_dual}

\begin{proof}
Recall the weak dual form of entropic OT
\[
\mathrm{EOT}_\varepsilon(p_0,p_1)
=\sup_f\Bigl\{\E_{x_1\sim p_1}[f(x_1)]-\varepsilon\,\E_{x_0\sim p_0}\bigl[\log Z(f,x_0)\bigr]\Bigr\},
\]
where $Z(f,x_0)=\int_{\R^D} \exp\!\bigl((f(x_1)-\tfrac12\|x_1-x_0\|^2)/\varepsilon\bigr)\,dx_1$.
By Proposition~\ref{lem:variational-Z} with $C=\tfrac D2\log(2\pi\varepsilon)$, for all $x_0\in\R^D$,
\[
\log Z(f,x_0)\le C-1+\xi(x_0)+\E_{z\sim\mathcal{N}(0,I)}
\Bigl[\exp\!\Bigl(\frac{f(x_0+\sqrt\varepsilon z)}{\varepsilon}-\xi(x_0)\Bigr)\Bigr].
\]
Plugging this upper bound into the weak dual (and using that it holds pointwise in $x_0$) yields
\[
\mathrm{EOT}_\varepsilon(p_0,p_1)\;\ge\;\sup_{f,\xi}\,\mathcal L(f,\xi).
\]
Conversely, since the inequality holds for every $\xi$, for each fixed $f$ we have
$\mathcal L(f,\xi)\le \E_{p_1}[f]-\varepsilon\E_{p_0}[\log Z(f,x_0)]$, hence
$\sup_{f,\xi}\mathcal L(f,\xi)\le \mathrm{EOT}_\varepsilon(p_0,p_1)$.
Finally, the bound is tight at $\xi(x_0)+C=\log Z(f,x_0)$ by Proposition~\ref{lem:variational-Z}, so equality holds.
Defining $\mathcal L(f):=\sup_\xi\mathcal L(f,\xi)$ gives $\sup_f\mathcal L(f)=\sup_{f,\xi}\mathcal L(f,\xi)$.
\end{proof}

\subsection{Proof of Theorem \ref{prop:fxiequality}}
\begin{proof}
Let the $\sup$ in the dual problem \eqref{eq:dual-final} be taken over a function $f^*\in L_1(p_1)$ (non-continuous in general), and set
\(
\xi^*(x_0):= \xi_{f^*}.
\)
Note that $\pi^*=\pi^{f^*}=\pi^{f^*,\xi^*}$.
Therefore,
\begin{align*}
\varepsilon\int\log\!\left(\frac{d\pi^*}{d\pi^{f,\xi}}\right)\,d\pi^*
=
\varepsilon\int_{X_0}( \xi(x_0)-\xi^*(x_0))\,dp_0(x_0)
+\int_{X_1} \bigl(f^*(x_1)-f(x_1)\bigr)\,dp_1(x_1),
\end{align*}
where we used that $\pi^*$ has marginals $p_0$ and $p_1$.

Next, we compute the total mass of $\pi^{f,\xi}$:
\[
\int_{\R^D}\int_{\R^D} \frac{p_0(x)}{(2\pi\varepsilon)^{D/2}}
\exp\!\left(\frac{f(x_1)-\frac12\|x_0-x_1\|^2}{\varepsilon}-\xi(x_0)\right)\,dx_1\,dx_0
=
\int_{X_0} \frac{Z_f(x_0)}{(2\pi\varepsilon)^{D/2}\exp\xi(x_0)}\,dp_0(x_0).
\]
Since $\pi^*$ is a probability measure. Combining the pieces, we obtain
\begin{align*}
\varepsilon\,\mathrm{KL}(\pi^*\|\pi^{f,\xi})=
\left(\int_{\R^D} f^*\,dp_1-\varepsilon\int_{\R^D} \xi^*\,dp_0\right)
-\left(\int_{\R^D} f\,dp_1-\varepsilon\int_{\R^D} \xi\,dp_0-\varepsilon\int_{\R^D}\frac{Z_f}{(2\pi\varepsilon)^{D/2}\exp\xi}\,dp_0+\varepsilon\right).
\end{align*}
Finally, since $(2\pi\varepsilon)^{D/2}\exp\xi^*=Z_{f^*}$, the first bracket equals $\mathcal{L}^*+\frac{\varepsilon D}{2}\log(2\pi\varepsilon)$, while the second bracket is precisely $\mathcal{L}(f,\xi)+\frac{\varepsilon D}{2}\log(2\pi\varepsilon)$ as defined in
\eqref{eq:final-loss}. Hence,
\[
\varepsilon\,\mathrm{KL}\!\left(\pi^*\,\middle\|\,\pi^{f,\xi}\right)
=
\mathcal{L}^*-\mathcal{L}(f,\xi),
\]
which proves \eqref{eq:excess-kl}.
\end{proof}

For the reader’s convenience, we briefly recall the empirical loss
\begin{equation}
\widehat{\mathcal{L}}(f, \xi) 
=
\varepsilon - \frac{\varepsilon D}{2}\log(2 \pi \varepsilon)+
\frac{1}{N_1} \sum_{j=1}^{N_1}{f(x_j^1)}
-
\frac{\varepsilon}{N_0} \sum_{i=1}^{N_0}
\xi(x_i^0)
-
\frac{\varepsilon}{N_0K}
\sum_{i,k=1}^{N_0,K}
\exp\!\left(
\frac{f(x_i^0 + \sqrt{\varepsilon}\, z_{i,k})}{\varepsilon} - \xi(x_i^0)
\right),
\end{equation}
where the samples are taken from the corresponding distributions $x^1_j \sim p_1,\; x^0_i \sim p_0,\; z_{i,k}\sim \mathcal{N}(0,I)$.

\subsection{Proof of Proposition \ref{prop:est-approx-decay}}
\begin{proof}We assume that all weight and bias parameters belong to closed and bounded
sets, hence the parameter sets
of $\cF$ and $\Xi$ are compact.
 Let 
 $$
 (\hat {f},\hat\xi)=\argmax_{f\in\mathcal{F},\xi\in\Xi}\; \widehat{\mathcal{L}}\left(f,\xi\right),
 $$
 and recall the notation $\xi_f=\log Z_{f}-\frac{D}{2}\log(2\pi\varepsilon)$.  
    Since \(\xi_f\) maximizes \(\mathcal L(f,\xi)\) over \(\xi\) for fixed \(f\), we have
\[
\mathcal L(\hat f,\xi_{\hat f})\ge \mathcal L(\hat f,\hat\xi).
\]
Therefore, by Theorem \ref{prop:fxiequality}
$$
\varepsilon\mathrm{KL}\!\left(\pi^*\middle\|\pi^{\hat f}\right)
=
\mathcal L^* - \mathcal L(\hat f,\xi_{\hat f})
\le
\mathcal L^* -\mathcal L(\hat f,\hat\xi)
=
\mathcal L^*-\sup_{f\in\mathcal F,\ \xi\in\Xi}\mathcal L(f,\xi)+\underbrace{\sup_{f\in\mathcal F,\ \xi\in\Xi}\mathcal L(f,\xi)-\mathcal L(\hat f,\hat\xi)}_{(*)}.
$$
$$
(*) = \sup_{f\in\cF,\xi\in\Xi} \mathcal L(f,\xi) - \sup_{f\in\cF,\xi\in\Xi} \widehat{ \mathcal L}(f,\xi) + \widehat{ \mathcal L}(\hat f,\hat \xi)- \mathcal L(\hat f,\hat \xi)
\le
\sup_{f\in\cF,\xi\in\Xi}\left[ \mathcal L(f,\xi) - \widehat{ \mathcal L}(f,\xi)\right]
+ 
\underbrace{\widehat{ \mathcal L}(\hat f,\hat \xi)- \mathcal L(\hat f,\hat \xi)}_{\le\sup_{f\in\cF,\xi\in\Xi}\left|\mathcal L(f,\xi)-\widehat{\mathcal L}(f,\xi)\right|}
$$
Combining the above estimates yields
\[
\varepsilon\,\mathrm{KL}\!\left(\pi^\ast\middle\|\pi^{\hat f}\right)
\le
\mathcal L^* -\sup_{f\in\mathcal F,\ \xi\in\Xi}\mathcal L(f,\xi)
+
2\sup_{f\in\mathcal F,\ \xi\in\Xi}
\left|\mathcal L(f,\xi)-\widehat{\mathcal L}(f,\xi)\right|.
\]
Taking expectation completes the proof.
\end{proof}

\subsection{Proof of Theorem  \ref{thm:estim}}

\begin{proof}

Let $M_\cF$ and $M_\Xi$ be clipping constants in classes $\cF$ and $\Xi$ respectively. 

Set
\[
h_{f,\xi}(x,z):=
\exp\!\left(\frac{f(x+\sqrt\varepsilon z)}{\varepsilon}-\xi(x)\right),
\qquad
g_{f,\xi}(x):=\mathbb E_{z\sim\mathcal N(0,I)} h_{f,\xi}(x,z).
\]
Then
\[
\mathcal L(f,\xi)-\widehat{\mathcal L}(f,\xi)
=
\Bigl(\mathbb E_{p_1}f-\frac1{N_1}\sum_{j=1}^{N_1}f(x_j^1)\Bigr)
-\varepsilon
\Bigl(\mathbb E_{p_0}\xi-\frac1{N_0}\sum_{i=1}^{N_0}\xi(x_i^0)\Bigr)
-\varepsilon
\Bigl(\mathbb E_{p_0}g_{f,\xi}-\frac1{N_0K}\sum_{i,k} h_{f,\xi}(x_i^0,z_{ik})\Bigr).
\]
Hence, by the triangle inequality,
\[
\sup_{f\in\mathcal F,\ \xi\in\Xi}
\left|
\mathcal L(f,\xi)-\widehat{\mathcal L}(f,\xi)
\right|
\le
\Delta_1+\varepsilon \Delta_2+\varepsilon \Delta_3+\varepsilon \Delta_4,
\]
where
\[
\Delta_1:=
\sup_{f\in\mathcal F}
\left|
\mathbb E_{p_1}f-\frac1{N_1}\sum_{j=1}^{N_1}f(x_j^1)
\right|,
\quad
\Delta_2:=
\sup_{\xi\in\Xi}
\left|
\mathbb E_{p_0}\xi-\frac1{N_0}\sum_{i=1}^{N_0}\xi(x_i^0)
\right|,
\]
\[
\Delta_3:=
\sup_{f,\xi}
\left|
\mathbb E_{p_0}g_{f,\xi}-\frac1{N_0}\sum_{i=1}^{N_0}g_{f,\xi}(x_i^0)
\right|,
\quad
\Delta_4:=
\sup_{f,\xi}
\left|
\frac1{N_0}\sum_{i=1}^{N_0}g_{f,\xi}(x_i^0)
-\frac1{N_0K}\sum_{i=1}^{N_0}\sum_{k=1}^{K}h_{f,\xi}(x_i^0,z_{ik})
\right|.
\]

By Lemma \ref{lem:sym},
\[
\E[\Delta_1]\le 2\,\mathfrak R_{N_1}(\cF,p_1),
\quad
\E[\Delta_2]\le 2\,\mathfrak R_{N_0}(\Xi,p_0),
\]
where all bounds are taken in expectation with respect to the random samples.

Recall that the classes $\mathcal F$ and $\Xi$ consist of feedforward neural networks
with depth at most $n$, with uniformly bounded layer norms, and with output clipping. Also the hidden activation $\sigma$ be $L_\sigma$-Lipschitz and satisfies
$\sigma(0)=0$, and that the layer matrices satisfy the row-wise constraints from \citep[Theorem~2]{GolowichRakhlinShamir2018}.
Then there exists a constants $C^{p_1}_{\mathcal F}>0$ and $C^{p_0}_{\Xi}>0$ such that
\begin{equation}\label{eq:Nsqrt_asymptotic}
\mathbb E[\Delta_1]
\le
\frac{C^{p_1}_{\mathcal F}}{\sqrt{N_1}},
\quad
\mathbb E[\Delta_2]
\le
\frac{C^{p_0}_{\Xi}}{\sqrt{N_0}}.
\end{equation}

\paragraph{Estimation of $\Delta_3$.} Let $\gamma_\varepsilon=\mathcal N(0,\varepsilon I)$ and let $\rho_\varepsilon(y)$ be its density. Consider the classes (note that $\gamma_1=\mathcal{N}(0,I)$)
\[
\mathcal{Z}:=\left\{z_f:=\int
\exp\!\left(\frac{f(x+\sqrt{\varepsilon}\,z)}{\varepsilon}\right)d\gamma_1(z):\ f\in\cF\right\},
\quad
\mathcal E:=\{e^{-\xi}:\ \xi\in\Xi\},
\quad
\mathcal G:=\{g_{f,\xi}=e^{-\xi}z_f:\ f\in\mathcal F,\ \xi\in\Xi\}.
\]
By Lemma \ref{lem:sym},
\[
\mathbb E[\Delta_3]\le 2\,\mathfrak R_{N_0}(\mathcal G,p_0).
\]

For $x\in X_0$, define \(
q_x(y):=\rho_\varepsilon(y-x)/\rho_\varepsilon(y)\)
and for $f\in\mathcal F$, define
\[
\phi_f(y):=\exp\!\left(\frac{f(y)}{\varepsilon}\right).
\]
Then
\[
z_f(x)
=
\int_{\mathbb R^D}\phi_f(y)\, \rho_\varepsilon(y-x)\,dy
=
\int_{\mathbb R^D}\phi_f(y)\, q_x(y)\, d\gamma_\varepsilon(y)
=
\langle \phi_f,q_x\rangle_{L^2(\gamma_\varepsilon)}.
\]
Since $|f|\le M_{\cF}$, we have
\[
\|\phi_f\|_{L^2(\gamma_\varepsilon)}
\le e^{M_\cF/\varepsilon}, \qquad \|q_x\|_{L^2(\gamma_\varepsilon)}^2 = \exp\left(\frac{\|x\|^2}{\varepsilon}\right)
\]
Therefore,
\[
\sup_{x\in \mathrm{supp}(p_0)}\|q_x\|_{L^2(\gamma_\varepsilon)}
\le
\exp\!\left(\frac{R_0^2}{2\varepsilon}\right),\quad \text{where } R_0= \max_{x\in\mathrm{supp}(p_0)} \|x\| 
\]
Using the representation above and given $\{\sigma_i\}_{i=1}^{N_0}$ -- i.i.d. Rademacher signs 
\[
\frac{1}{N_0}\sum_{i=1}^{N_0} \sigma_i z_f(x_i)
=
\left\langle
\phi_f,\,
\frac{1}{N_0}\sum_{i=1}^{N_0} \sigma_i q_{x_i}
\right\rangle_{L^2(\gamma_\varepsilon)}.
\]
Hence, by Cauchy--Schwarz,
\[
\mathfrak R_{N_0}(\mathcal Z,p_0)
\le
e^{M_\cF/\varepsilon}\,
\mathbb E_{x,\sigma}\left[
\left\|
\frac{1}{N_0}\sum_{i=1}^{N_0} \sigma_i q_{x_i}
\right\|_{L^2(\gamma_\varepsilon)}
\right].
\]
Applying Jensen's inequality, the independence of $\sigma_i$ and $\E\sigma_i=0$ of Rademacher signs, we get
\[
\mathfrak R_{N_0}(\mathcal Z,p_0)
\le
e^{M_\cF/\varepsilon}
\left(
\mathbb E
\left\|
\frac{1}{N_0}\sum_{i=1}^{N_0} \sigma_i q_{x_i}
\right\|_{L^2(\gamma_\varepsilon)}^2
\right)^{1/2}
=
e^{M_\cF/\varepsilon}
\left(
\frac{1}{N_0^2}
\sum_{i=1}^{N_0}
\mathbb E\|q_{x_i}\|_{L^2(\gamma_\varepsilon)}^2
\right)^{1/2}
\le
\frac{e^{M_\cF/\varepsilon+R_0^2/(2\varepsilon)}}{\sqrt N_0}.
\]

Note that $|\xi|\le M_\Xi$, the function \(\phi(t):=e^{-t}\) is $e^{M_\Xi}$-Lipschitz on $[-M_\Xi,M_\Xi]$, and
\(
\mathfrak R_{N_0}(\mathcal E,p_0)
=
\mathfrak R_{N_0}(\phi\circ \Xi,p_0).
\)
Hence, by Lemma \ref{lem:vector-contraction} and \eqref{eq:Nsqrt_asymptotic},
\[
\mathfrak R_{N_0}(\mathcal E,p_0)
\le
\sqrt{2}\,e^{M_\Xi}\,\mathfrak R_{N_0}(\Xi,p_0)
\le
\frac{\sqrt{2}\,e^{M_\Xi} C^{p_0}_\Xi}{\sqrt{N_0}}.
\]
Note that the classes $\mathcal E$ and $\mathcal Z$ are uniformly bounded
\[
e^{-M_\Xi}\le e^{-\xi(x)}\le e^{M_\Xi},
\qquad
0\le z_f(x)\le e^{M_\cF/\varepsilon}.
\]
Consider the map $\psi(u,v):=uv$ on the rectangle $[0,e^{M_\cF/\varepsilon}]\times [e^{-M_\Xi},e^{M_\Xi}]$.
It is coordinatewise Lipschitz with constants
$L_u=e^{M_\Xi}$ and $L_v=e^{M_\cF/\varepsilon}$.
Applying Lemma \ref{lem:2d-contract} to
\[
\mathcal Z\times \mathcal E
=
\{(z_f,e^{-\xi}): f\in\mathcal F,\ \xi\in\Xi\},
\]
we obtain
\[
\mathfrak R_{N_0}(\mathcal G,p_0)
\le
{2}e^{M_\Xi}\,\mathfrak R_{N_0}(\mathcal Z,p_0)
+
{2}e^{M_\cF/\varepsilon}\,\mathfrak R_{N_0}(\mathcal E,p_0).
\]
Combining this with the bounds above yields
\[
\mathfrak R_{N_0}(\mathcal G,p_0)
\le
\frac{2e^{M_\cF/\varepsilon+M_\Xi}}{\sqrt{N_0}}\left(e^{R_0^2/2\varepsilon}
+
\sqrt{2}\,C^{p_0}_\Xi\right).
\]

\paragraph{Estimation of $\Delta_4$.} 
 Fix $x\in X_0$ and set
\[
\cF_x:=\{z\mapsto f(x+\sqrt{\varepsilon}z): f\in\cF\}.
\]
Since $\xi(x)$ does not depend on $z$ and $|\xi|\le M_\Xi$, we have
\[
\sup_{f,\xi}
\left|
\E_z h_{f,\xi}(x,z)-\frac1K\sum_{k=1}^K h_{f,\xi}(x,z_k)
\right|
\le
e^{M_\Xi}
\sup_{f\in\cF}
\left|
\E_z e^{f(x+\sqrt{\varepsilon}z)/\varepsilon}
-\frac1K\sum_{k=1}^K e^{f(x+\sqrt{\varepsilon}z_k)/\varepsilon}
\right|.
\]
By symmetrization and the contraction inequality applied to
$t\mapsto e^{t/\varepsilon}$ on $[-M_\cF,M_\cF]$, whose Lipschitz constant is
$e^{M_\cF/\varepsilon}/\varepsilon$, we get
\[
\E_z\sup_{f,\xi}
\left|
\E_z h_{f,\xi}(x,z)-\frac1K\sum_{k=1}^K h_{f,\xi}(x,z_k)
\right|
\le
\frac{C e^{M_\Xi+M_\cF/\varepsilon}}{\varepsilon}
\mathfrak R_K(\cF_x,\gamma_1),
\]
where $\gamma_1=\mathcal N(0,I)$ and $C>0$ is a universal constant.

We now use only the sample-dependent Rademacher bound from
\citep[Theorem~1]{GolowichRakhlinShamir2018}. In the notation adapted to our
network class, it gives a constant $A_\cF>0$, depending only on the depth and
the layer-norm bounds of $\cF$, such that for every deterministic sample
$y_1,\dots,y_K\in\mathbb R^D$,
\[
\widehat{\mathfrak R}_{\left\{y_k\right\}_{k=1}^K}(\cF)
\le
\frac{A_\cF}{K}
\left(\sum_{k=1}^K \|y_k\|^2\right)^{1/2}.
\]
Applying this inequality to $y_k=x+\sqrt{\varepsilon}z_k$ yields
\[
\mathfrak R_K(\cF_x,\gamma_1)
\le
\frac{A_\cF}{K}
\E\left(\sum_{k=1}^K\|x+\sqrt{\varepsilon}Z_k\|^2\right)^{1/2}.
\]
By Jensen's inequality,
\[
\E\left(\sum_{k=1}^K\|x+\sqrt{\varepsilon}Z_k\|^2\right)^{1/2}
\le
\left(K\E\|x+\sqrt{\varepsilon}Z\|^2\right)^{1/2}
=
\sqrt K\,(\|x\|^2+\varepsilon D)^{1/2}.
\]
Since $X_0$ is compact, $R_0:=\sup_{x\in X_0}\|x\|<\infty$, and therefore
\[
\sup_{x\in X_0}\mathfrak R_K(\cF_x,\gamma)
\le
\frac{A_\cF (R_0^2+\varepsilon D)^{1/2}}{\sqrt K}.
\]
Consequently,
\[
\E[\Delta_4]
\le
\frac{C e^{M_\Xi+M_\cF/\varepsilon}}{\varepsilon}
\frac{A_\cF (R_0^2+\varepsilon D)^{1/2}}{\sqrt K}.
\]

Combining the bounds for $\Delta_1,\Delta_2,\Delta_3,\Delta_4$, we obtain
constants $C_1,C_2,C_3>0$, depending only on $\varepsilon$, the clipping levels,
the network-complexity parameters, and the support bounds, such that
\[
\E\sup_{f\in\cF,\xi\in\Xi}
\left|
\mathcal L(f,\xi)-\widehat{\mathcal L}(f,\xi)
\right|
\le
\frac{C_1}{\sqrt{N_1}}
+
\frac{C_2}{\sqrt{N_0}}
+
\frac{C_3}{\sqrt K}.
\]
This proves the claimed
$O(N_1^{-1/2})+O(N_0^{-1/2})+O(K^{-1/2})$ bound.

\end{proof}

\subsection{Proof of Theorem \ref{thm:approx}}

\begin{proof}

It is sufficient to take the supremum in the weak dual problem \eqref{eq:dual-final} over continuous potentials. Indeed, under the standing admissibility assumptions, every admissible measurable potential can first be approximated, in the value of the functional, by its bounded truncations. For bounded potentials, approximation by continuous functions in the weighted space naturally associated with the Gaussian kernel in $Z_f$ preserves both terms of
$$
\mathcal L(f)
=
\int f(x_1)\,dp_1(x_1)
-
\varepsilon\int \log Z_f(x_0)\,dp_0(x_0).
$$
The first term is stable by $L^1(p_1)$-convergence, and the log-partition term is stable because the exponential is uniformly Lipschitz on bounded ranges. Hence the value of the supremum is unchanged if the admissible measurable class is replaced by $C(\mathbb R^D)$.

Fix $\delta>0$. Then there exist $f\in C(\R^D)$ such that
\[
\mathcal L^*-\mathcal L(f,\xi_f)<\frac{\delta}{3}.
\]
Recall that 
\[
\xi_f(x_0)
=
\log \mathbb E_{z\sim\mathcal N(0,I)}
\exp\!\left(\frac{f(x_0+\sqrt\varepsilon z)}{\varepsilon}\right),
\]
so $\xi_f$ is continuous in our setting. Further in the text, we will omit the index at $\xi_f$ and set $\xi:=\xi_f$.
Denote
\[
\Phi(f,\xi)
:=
\mathbb E_{x_0\sim p_0}\mathbb E_{z\sim\mathcal N(0,I)}
\exp\!\left(\frac{f(x_0+\sqrt\varepsilon z)}{\varepsilon}-\xi(x_0)\right).
\]
Then
\[
\mathcal L(f,\xi)
=
\varepsilon\Bigl(1-\frac D2\log(2\pi\varepsilon)\Bigr)
+\mathbb E_{p_1}[f]
-\varepsilon\mathbb E_{p_0}[\xi]
-\varepsilon\,\Phi(f,\xi).
\]

Since $f$ and $\xi$ are continuous and $X_1,X_0$ are compact, the restrictions
$f|_{X_1}$ and $\xi|_{X_0}$ are bounded. Choose $M>0$ so that
\[
M>\|f\|_{L^\infty(X_1)}+\|\xi\|_{L^\infty(X_0)}
\qquad\text{and}\qquad
\varepsilon e^{\|\xi\|_{L^\infty(X_0)}}e^{-M/\varepsilon}<\frac{\delta}{6}.
\]
Let
\[
T_M(t):=(-M)\vee t\wedge M,
\qquad
\bar f:=T_M(f).
\]
Then $\bar f\in C(\mathbb R^D)$ and $|\bar f|\le M$ on $\mathbb R^D$. Moreover,
$\bar f=f$ on $X_1$, hence the linear term $\mathbb E_{p_1}[f]$ does not change.
Also, clipping from above can only decrease the exponential term in $\Phi$, while clipping
from below changes it pointwise by at most
$
e^{\|\xi\|_{L^\infty(X_0)}}e^{-M/\varepsilon}.
$
Therefore
\[
\Phi(\bar f,\xi)\le \Phi(f,\xi)+e^{\|\xi\|_{L^\infty(X_0)}}e^{-M/\varepsilon},
\]
and so
\[
\mathcal L(\bar f,\xi)
\ge
\mathcal L(f,\xi)
-
\varepsilon e^{\|\xi\|_{L^\infty(X_0)}}e^{-M/\varepsilon}
>
\mathcal L^*-\frac{\delta}{2}.
\]
Next define the probability measure
\[
q(dx_1):=
\left(
\int_{X_0}(2\pi\varepsilon)^{-D/2}
\exp\!\left(-\frac{\|x_1-x_0\|^2}{2\varepsilon}\right)p_0(dx_0)
\right)dx_1.
\]
Let
\[
L_M:=e^{M/\varepsilon+M}\max\{\varepsilon^{-1},1\}.
\]
If $|f_1|,|f_2|\le M$ on $\mathbb R^D$ and $|\xi_1|,|\xi_2|\le M$ on $X_0$, then
\[
\bigl|\Phi(f_1,\xi_1)-\Phi(f_2,\xi_2)\bigr|
\le
L_M
\left(
\int_{\mathbb R^D}|f_1-f_2|\,dq
+
\int_{X_0}|\xi_1-\xi_2|\,dp_0
\right).
\]
Consequently,
\begin{equation}\label{eq:loss-stability-corrected}
\begin{aligned}
|\mathcal L(f_1,\xi_1)-\mathcal L(f_2,\xi_2)|
\le\;&
\int_{X_1}|f_1-f_2|\,dp_1
+\varepsilon\int_{X_0}|\xi_1-\xi_2|\,dp_0
\\
&\quad
+\varepsilon L_M
\left(
\int_{\mathbb R^D}|f_1-f_2|\,dq
+
\int_{X_0}|\xi_1-\xi_2|\,dp_0
\right).
\end{aligned}
\end{equation}

Choose $\alpha>0$ and then $R>0$ so large that
\[
X_1\subset B_R
\qquad\text{and}\qquad
q(\R^d \setminus B_R)<\alpha.
\]
Since $\sigma$ is Lipschitz, it is continuous; since it is also non-polynomial,
the universal approximation theorem for one-hidden-layer ridge networks with biases applies, see \citep[Theorem 3.1]{Pinkus1999}. More precisely, the theorem concerns the family
\[
M(\sigma)
=
\mathrm{span}\bigl\{x\mapsto \sigma(w^\top x+b):\; w\in\mathbb R^D,\ b\in\mathbb R\bigr\},
\]
which is dense in $C(K)$ for every compact $K\subset\mathbb R^D$.

Hence there exist one-hidden-layer networks
\[
g_\theta(x)=\sum_{k=1}^{m_1}a_k\,\sigma(w_k^\top x+b_k),
\qquad
h_\vartheta(x)=\sum_{\ell=1}^{m_2}c_\ell\,\sigma(u_\ell^\top x+d_\ell),
\]
such that
\[
\sup_{x\in B_R}|g_\theta(x)-\bar f(x)|<\alpha,
\qquad
\sup_{x\in X_0}|h_\vartheta(x)-\xi(x)|<\alpha.
\]
Now perform clipping after approximation and define
\[
f_\theta:=T_M\circ g_\theta,
\qquad
\xi_\vartheta:=T_M\circ h_\vartheta.
\]
Since $|\bar f|\le M$ on $\mathbb R^D$ and $|\xi|\le M$ on $X_0$, clipping does not increase
the approximation error on the target sets; thus
\[
\sup_{x\in B_R}|f_\theta(x)-\bar f(x)|<\alpha,
\qquad
\sup_{x\in X_0}|\xi_\vartheta(x)-\xi(x)|<\alpha,
\]
and, by construction,
\[
|f_\theta(x)|\le M \quad \forall x\in\mathbb R^D.
\]
Therefore,
\[
\int_{X_1}|f_\theta-\bar f|\,dp_1\le \alpha,
\qquad
\int_{X_0}|\xi_\vartheta-\xi|\,dp_0\le \alpha,
\]
and
\[
\int_{\mathbb R^D}|f_\theta-\bar f|\,dq
\le
\alpha\,q(B_R)+2M\,q(\R^d \setminus B_R)
\le
\alpha+2M\alpha.
\]
Substituting these bounds into \eqref{eq:loss-stability-corrected}, we obtain
\[
|\mathcal L(f_\theta,\xi_\vartheta)-\mathcal L(\bar f,\xi)|
\le
\alpha+\varepsilon\alpha+\varepsilon L_M\bigl(2\alpha+2M\alpha\bigr).
\]
Choosing $\alpha>0$ sufficiently small, we get
\[
|\mathcal L(f_\theta,\xi_\vartheta)-\mathcal L(\bar f,\xi)|<\frac{\delta}{2}.
\]
Hence
\[
\mathcal L(f_\theta,\xi_\vartheta)
>
\mathcal L^*-\delta.
\]

Finally, choose neural-network classes
$\mathcal F$ and $\Xi$ so that they contain the clipped networks $f_\theta$ and $\xi_\vartheta$ constructed above. Then
\[
\sup_{f\in\mathcal F,\ \xi\in\Xi}\mathcal L(f,\xi)
\ge
\mathcal L(f_\theta,\xi_\vartheta)
>
\mathcal L^*-\delta.
\]
This proves the claim.

\end{proof}

\section{Details of the Experiments}
\label{sec-exp-details}

\subsection{VarEOT: Optimization and Architecture}
\label{sec-vareot-details}

\paragraph{Optimization.}
Training is performed using the AdamW \citep{loshchilov2017decoupled}
optimizer with a learning rate $\text{lr} = 10^{-4}$, momentum
parameters $\beta_1 = 0.7$ and $\beta_2 = 0.8$, and weight decay set to
$10^{-4}$. We additionally employ an exponential moving average (EMA) of
model parameters with momentum $\text{ema\_momentum} = 0.999$, which is
used for evaluation. All models are trained for $10^4$ gradient steps.

\paragraph{Network architecture.}
The transport potential is parameterized by a multilayer perceptron
(MLP). The network consists of four fully connected layers with SiLU
activations. The input dimensionality is denoted by $d_{\text{in}}$
(512 for latent-space experiments), the hidden layer width is fixed to
256 units, and the output dimension is $d_{\text{out}} = 1$.

\paragraph{Training setup.}
All models are trained with a batch size of 256. The number of Monte
Carlo samples used to estimate expectations is fixed to $K = 256$ in the image
experiments, for MSCI we use $K = 128$.

\paragraph{Langevin inference.}
The number of Langevin steps $S$ (denoted as NFE) and the corresponding
step size are chosen depending on the experimental setting.
For all toy experiments on the SwissRoll dataset, we use $1000$
Langevin steps with a fixed step size of $10^{-3}$.
For image-based experiments in the ALAE latent space, we consider
multiple inference budgets. The step sizes for each NFE configuration
are selected based on FID performance, as illustrated in
Figure~\ref{fig:fid-heatmaps}. Specifically, we use a step size of $0.5$
for $\text{NFE}=2$, a step size of $0.1$ for $\text{NFE}=10$, and a step
size of $10^{-3}$ for $\text{NFE}=1000$.
\paragraph{Update schedule and practical considerations.}
We perform simultaneous optimization of both the dual potential network $\hat{f}_\theta$ and the auxiliary network $\hat{\xi}_\psi$ at every training step, which we found to lead to stable training and strong empirical performance. To further stabilize training, we clip the difference values appearing in the exponential term of \eqref{eq:empirical-loss} from above by $30$, and apply standard gradient clipping with norm threshold $1$.
\subsection{Baseline Methods}
\label{sec-baselines-details}

For completeness, we summarize the main technical details of the
baseline methods used for comparison in the unpaired image-to-image
translation experiments.

\paragraph{LightSB.}
For LightSB, we follow the standard configuration reported in prior
work. The model is trained using $K = 10$ Gaussian components, a
learning rate of $\text{lr} = 10^{-3}$, and a batch size of 128.
Optimization is performed for $10^4$ gradient steps.

\paragraph{EgNOT.}

For EgNOT, we follow the optimization/hyperparameter setup from the original work, but adopt the same neural network architecture for potential $f$ as in section \ref{sec-vareot-details} for fair comparison with VarEOT. Training/inference hyperparameters were selected by grid search over a reasonable parameter space based on FID performance. For all ALAE experiments, Adam optimizer is used with $\text{lr} = 5\cdot10^{-5}$, $(\beta_1, \beta_2) = (0, 0.999)$. Optimization is performed for $10^4$ gradient steps, batch size is $128$.

\subsection{Image Data and Preprocessing}
\label{sec-data-details}
For image-based experiments, we rely on the official implementation of
ALAE and the corresponding pretrained models, available at
\begin{center}
\url{https://github.com/podgorskiy/ALAE}
\end{center}
To obtain semantic attributes for the FFHQ dataset, we use the neural
network–based annotations released at
\begin{center}
\url{https://github.com/DCGM/ffhq-features-dataset}
\end{center}
For baseline comparisons, we use the official implementations of LightSB
and EgNOT, available at
\begin{center}
\url{https://github.com/ngushchin/LightSB}
\end{center}
and
\begin{center}
\url{https://github.com/PetrMokrov/Energy-guided-Entropic-OT}
\end{center}
respectively.

\section{Additional Experimental Results}
\label{sec-exp-additional}

\subsection{Wall-Clock Training Time Comparison}\label{app:wall-clock}

We report wall-clock training times for the M$\to$F experiment ($\varepsilon{=}1.0$) with $10000$ iterations, using the same hyperparameters as in Appendix~\ref{sec-vareot-details}. For VarEOT, we vary the number of particles $K \in \{32, 64, 128, 256, 512\}$ to study the speed--quality trade-off. All experiments for VarEOT and EgNOT were conducted on a single GPU, while LightSB experiments were run on CPU, as in the original paper.

\begin{table}[h]
\centering
\begin{tabular}{lcrrr}
\toprule
\textbf{Method} & \textbf{Config} & \textbf{Time} & \textbf{Iter} & \textbf{FID} $\downarrow$ \\
\midrule
VarEOT & $K{=}512$ & $409$ s & $10000$ & $\mathbf{9.575}$ \\
VarEOT & $K{=}256$ & $215$ s & $10000$ & $\mathbf{10.55}$ \\
VarEOT & $K{=}128$ & $131$ s & $10000$ & $\mathbf{11.43}$ \\
VarEOT & $K{=}64$  & $85$ s  & $10000$ & $\mathbf{12.85}$ \\
VarEOT & $K{=}32$  & $85$ s  & $10000$ & $\mathbf{13.45}$ \\
\midrule
LightSB \citep{korotin2024light}  & --- & $191$ s  & $10000$ & $22.63$ \\
EgNOT \citep{mokrov2024energyguided} & --- & $1692$ s & $10000$ & $17.58$ \\
\bottomrule
\end{tabular}
\vspace{3mm}
\caption{Wall-clock training time and FID on the M$\to$F task ($\varepsilon{=}1.0$; $10000$ iterations). VarEOT training time scales with the number of particles $K$; larger $K$ yields better FID at the cost of longer training.}
\label{tab:wall-clock}
\end{table}

\subsection{Other Unpaired Image-to-Image Translation Setups}
In the main text, we present qualitative results for the unpaired
image-to-image translation task using the ALAE encoder in the
Male$\to$Female (M$\to$F) setup. In this appendix, we provide additional
qualitative examples for the same task under alternative and commonly
used translation directions.

Specifically, we report results for the Woman$\to$Man (F$\to$M),
Adult$\to$Child (A$\to$C), and Child$\to$Adult (C$\to$A) setups. All experiments are
conducted under the same training protocol and model configuration as
in the main text, and differ only in the choice of the source and target
domains. For consistency, we fix the regularization parameter to
$\varepsilon = 1.0$ for all setups shown below.

Figures~\ref{fig:W2M}--\ref{fig:Y2O} demonstrate that the proposed method
successfully learns meaningful transport maps across different
translation directions, further confirming that the qualitative
behavior observed in the M$\to$F setup generalizes to other unpaired
image-to-image translation scenarios.

\begin{figure}[h]
    \centering
    \includegraphics[width=\linewidth]{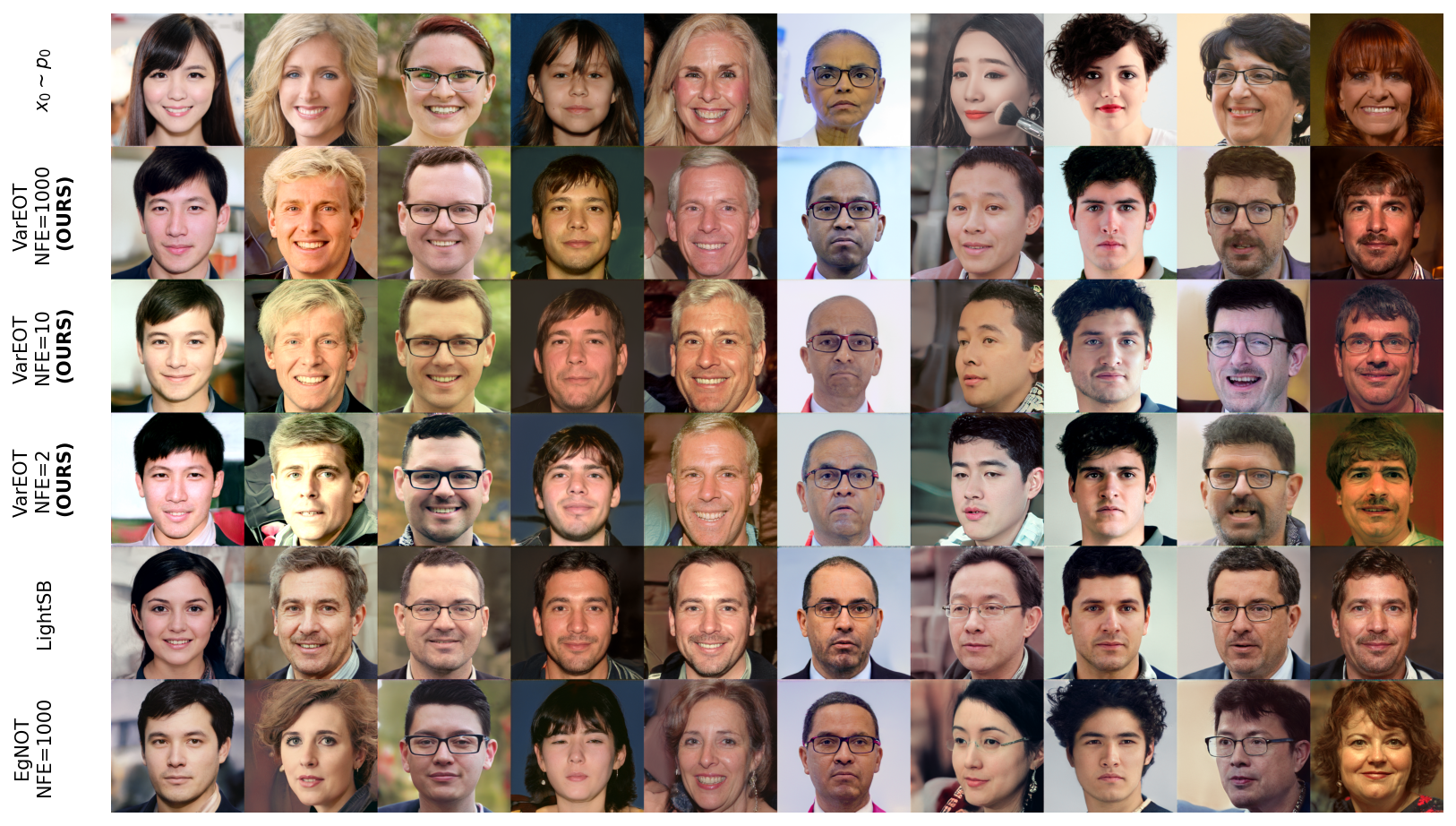}
    \caption{\centering
    Qualitative comparison for \textit{Female} $\rightarrow$ \textit{Male} translation with $\varepsilon=1.0$.
    From top to bottom: input samples, VarEOT (ours), LightSB, and EgNOT.
    Input images are selected from the test set: we take the first 300 samples and rank them by encoder-decoder reconstruction quality (LPIPS), displaying the top-ranked examples.}
    \label{fig:W2M}
\end{figure}

\begin{figure}[h]
    \centering
    \includegraphics[width=\linewidth]{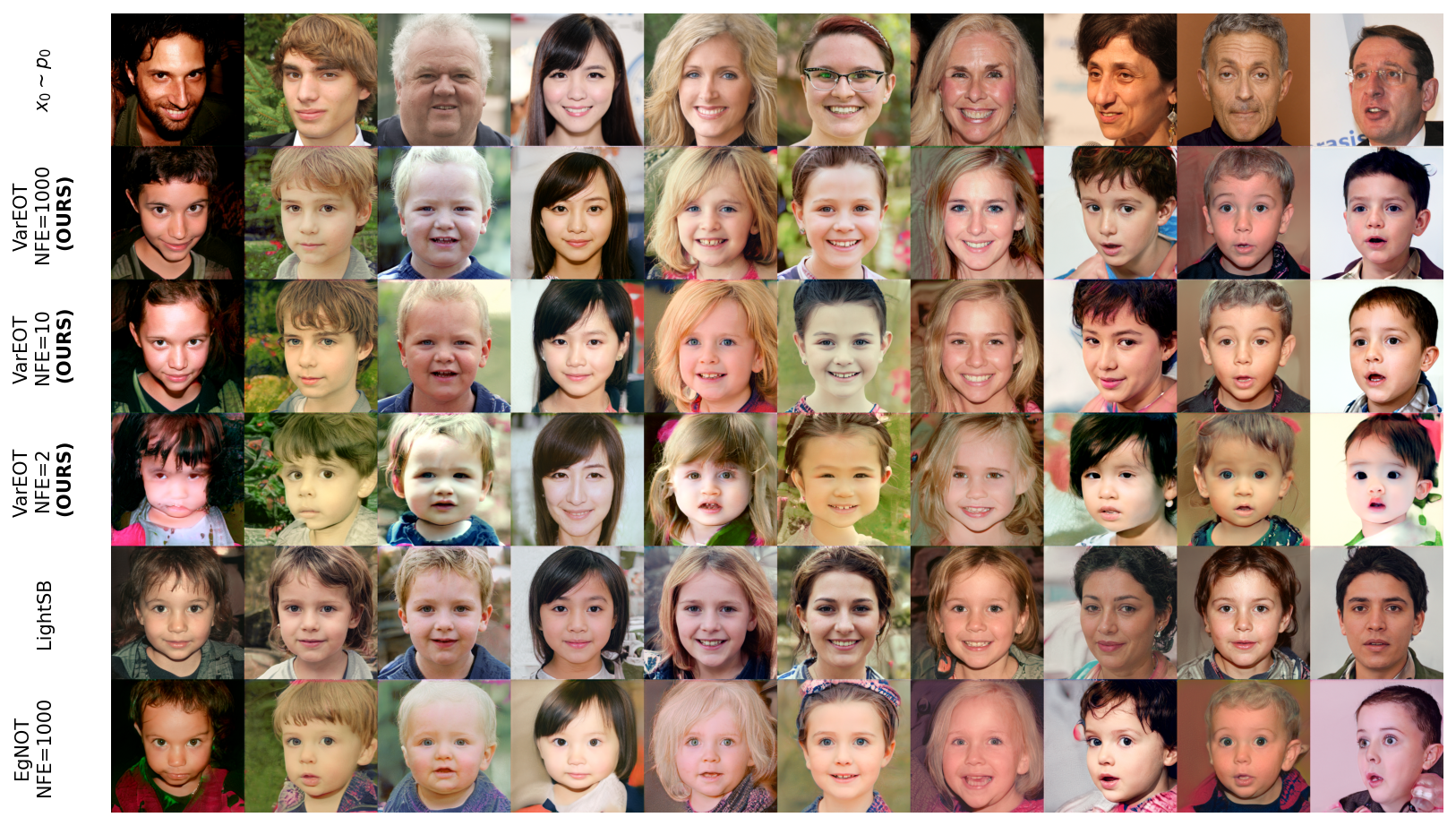}
    \caption{\centering
    Qualitative comparison for \textit{Adult} $\rightarrow$ \textit{Child} translation with $\varepsilon=1.0$.
    From top to bottom: input samples, VarEOT (ours), LightSB, and EgNOT.
    Input images are selected from the test set: we take the first 300 samples and rank them by encoder-decoder reconstruction quality (LPIPS), displaying the top-ranked examples.}
    \label{fig:O2Y}
\end{figure}

\begin{figure}[h]
    \centering
    \includegraphics[width=\linewidth]{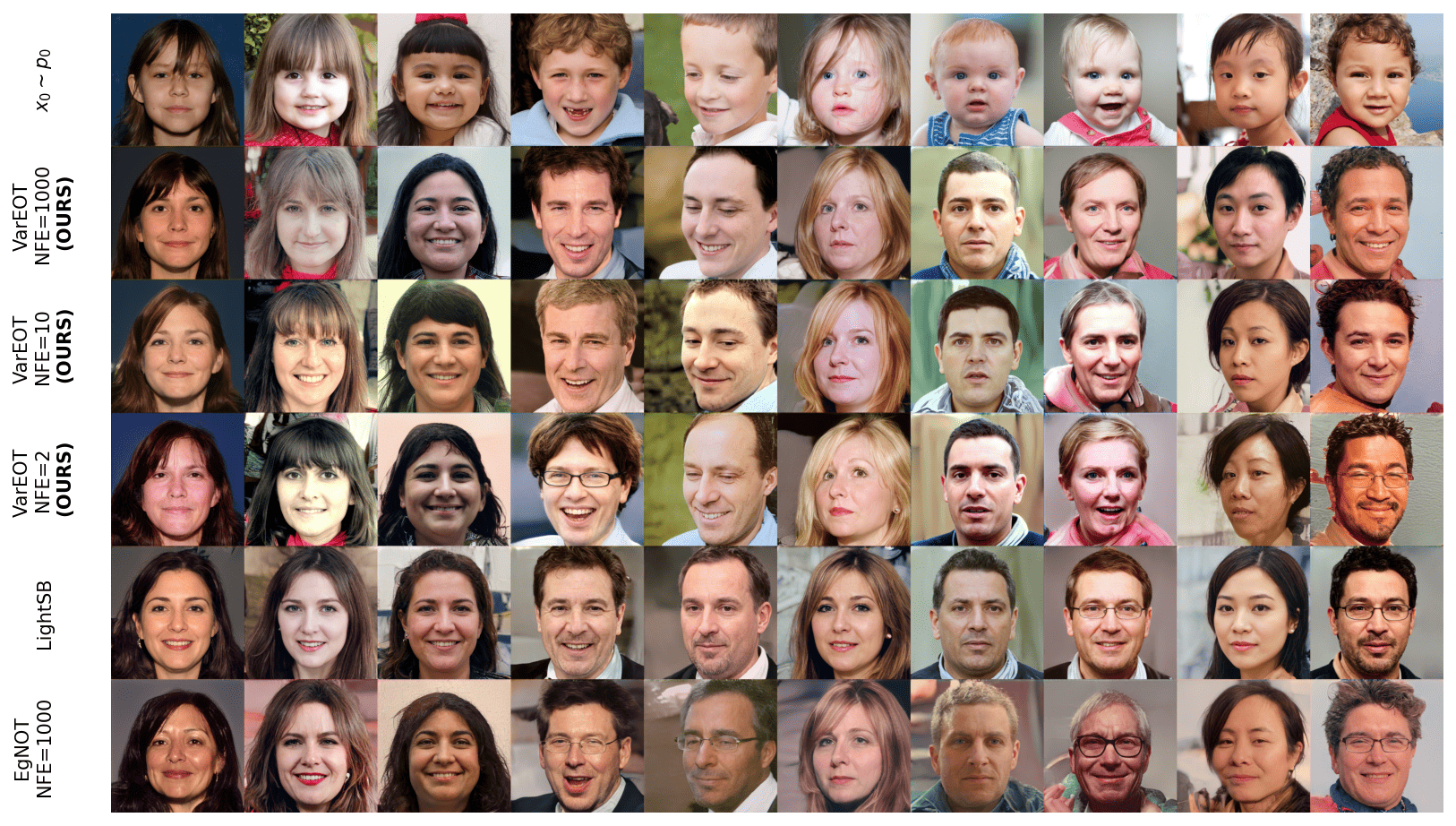}
    \caption{\centering
    Qualitative comparison for \textit{Child} $\rightarrow$ \textit{Adult} translation with $\varepsilon=1.0$.
    From top to bottom: input samples, VarEOT (ours), LightSB, and EgNOT.
    Input images are selected from the test set: we take the first 300 samples and rank them by encoder-decoder reconstruction quality (LPIPS), displaying the top-ranked examples.}
    \label{fig:Y2O}
\end{figure}

\subsection{Dependence on the parameter $\varepsilon$.}\label{app:eps study}
In Figure~\ref{fig:eps-dependence}, we show how the solution learned by VarEOT depends on the parameter $\varepsilon$ in the \textit{Male}$\rightarrow$\textit{Female} experiment. As expected, the diversity increases with the increase of $\varepsilon$.

\begin{figure*}[!t]
\vspace{-15mm}
\begin{subfigure}[b]{0.99 \linewidth}
\centering
\includegraphics[width=0.7\linewidth]{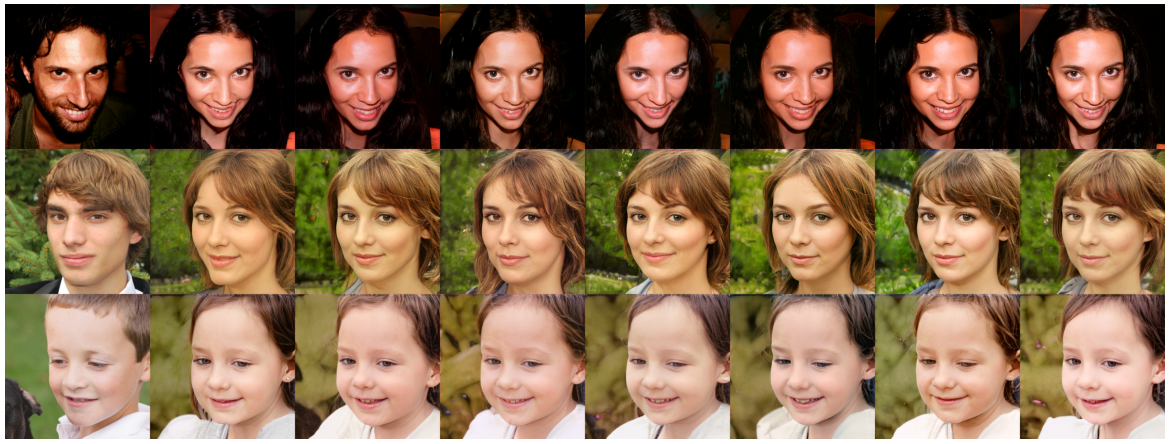}
\caption{\vspace{-1mm} \centering VarEOT \textit{Male} $\rightarrow$ \textit{Female}, $\varepsilon=0.1$. Almost no diversity.}
\vspace{-1mm} %
\end{subfigure}
\vskip\baselineskip
\begin{subfigure}[b]{0.99 \linewidth}
\centering
\includegraphics[width=0.7\linewidth]{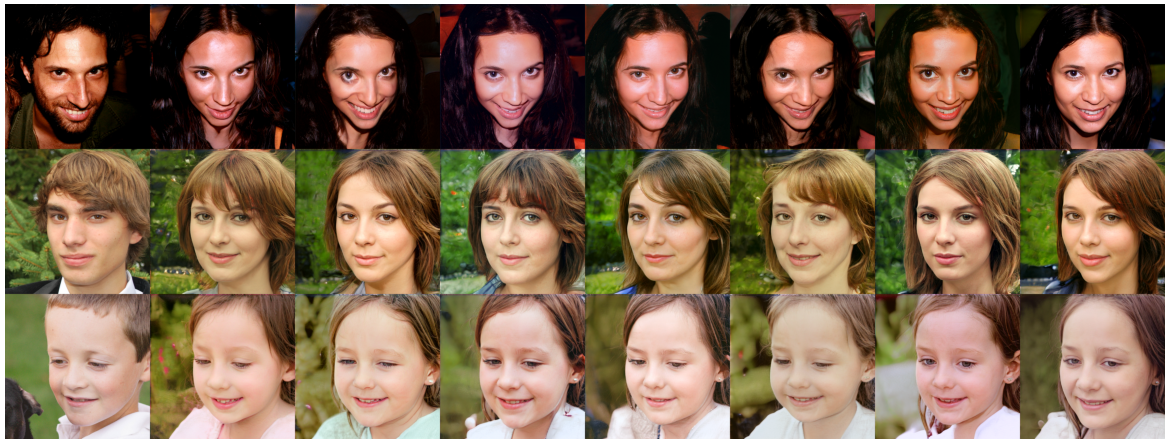}
\caption{\vspace{-1mm} \centering VarEOT \textit{Male} $\rightarrow$ \textit{Female}, $\varepsilon=0.5$. Reasonable diversity.}
\vspace{-1mm} %
\end{subfigure}
\vskip\baselineskip
\begin{subfigure}[b]{0.99 \linewidth}
\centering
\includegraphics[width=0.7\linewidth]{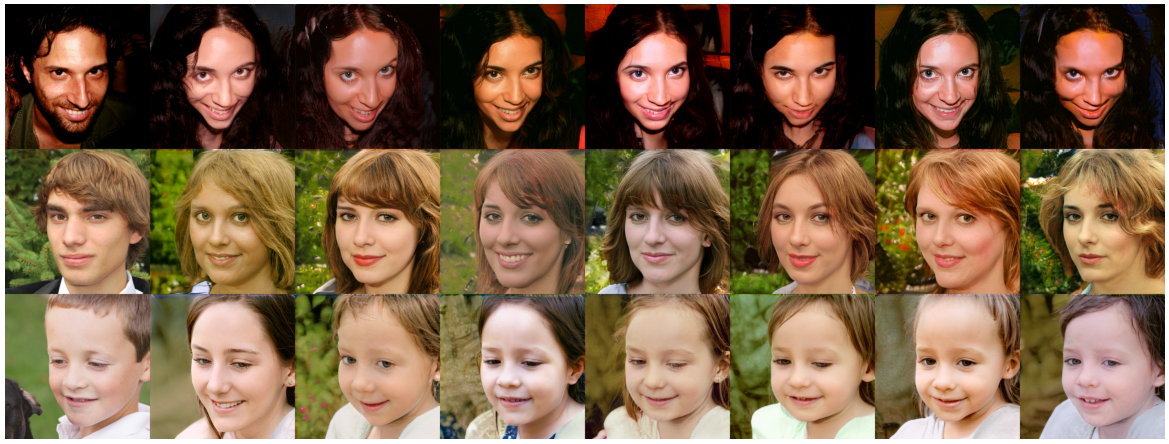}
\caption{\vspace{-1mm} \centering VarEOT \textit{Male} $\rightarrow$ \textit{Female}, $\varepsilon=1.0$. Moderate diversity.}
\vspace{-1mm} %
\end{subfigure}
\vskip\baselineskip
\begin{subfigure}[b]{0.99 \linewidth}
\centering
\includegraphics[width=0.7\linewidth]{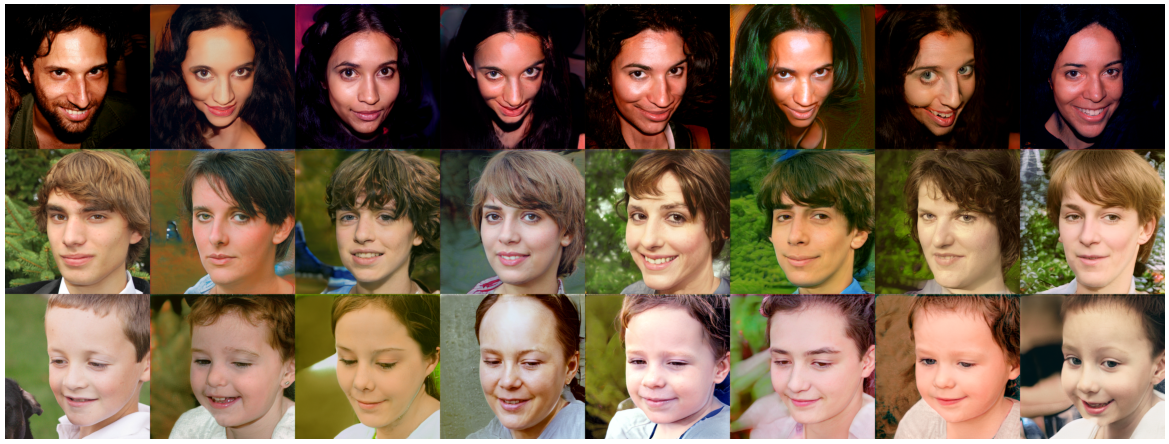}
\caption{\vspace{-1mm} \centering VarEOT \textit{Male} $\rightarrow$ \textit{Female}, $\varepsilon=10.0$. High diversity.}
\vspace{-1mm} %
\end{subfigure}
\caption{\centering VarEOT with NFE=10 in task: \textit{Man} $\rightarrow$ \textit{Woman} for different $\varepsilon$}\label{fig:eps-dependence}
\end{figure*}

\subsection{Effect of Langevin Inference Parameters}\label{app:lang study}

To further analyze the behavior of VarEOT at inference time, we study the
sensitivity of the translation quality to the Langevin sampling parameters in
the \textit{Male}$\rightarrow$\textit{Female} (M$\to$F) image-to-image translation
task. Figure~\ref{fig:fid-heatmaps} reports heatmaps of FID scores as a function of
the Langevin step size and the number of sampling steps, for different values of
the entropic regularization parameter $\varepsilon$.

\begin{figure*}[t]
    \centering
    \begin{subfigure}{0.48\textwidth}
        \centering
        \includegraphics[width=\linewidth]{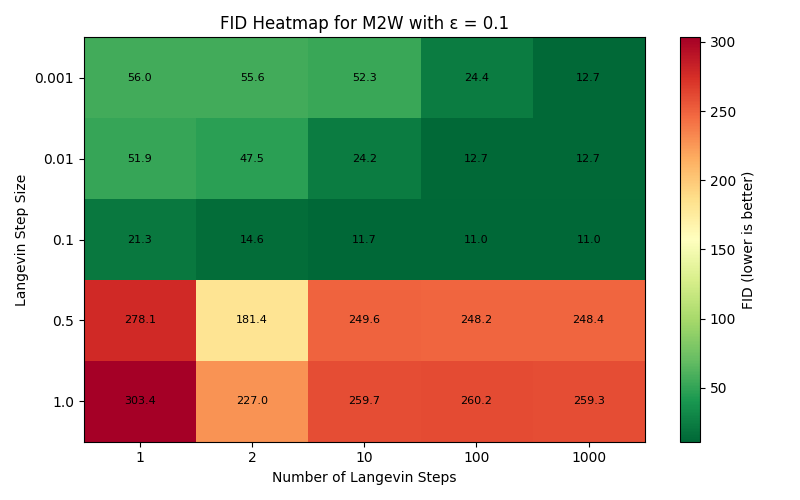}
        \caption{$\varepsilon = 0.1$}
    \end{subfigure}
    \hfill
    \begin{subfigure}{0.48\textwidth}
        \centering
        \includegraphics[width=\linewidth]{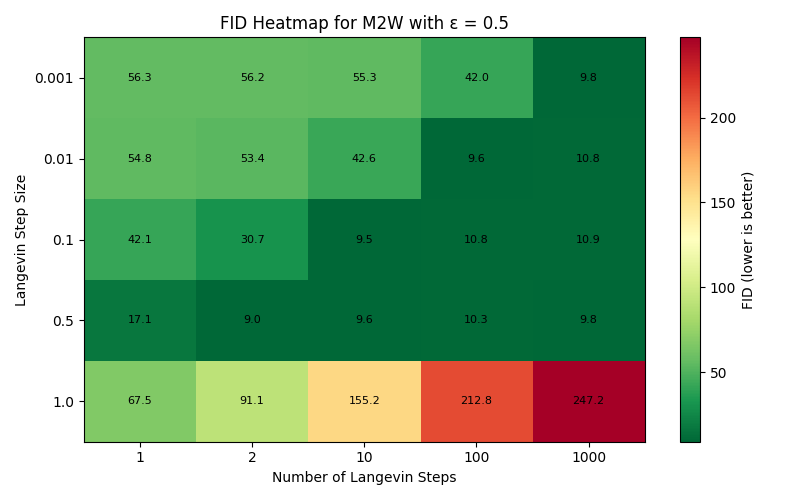}
        \caption{$\varepsilon = 0.5$}
    \end{subfigure}

    \vspace{2mm}

    \begin{subfigure}{0.48\textwidth}
        \centering
        \includegraphics[width=\linewidth]{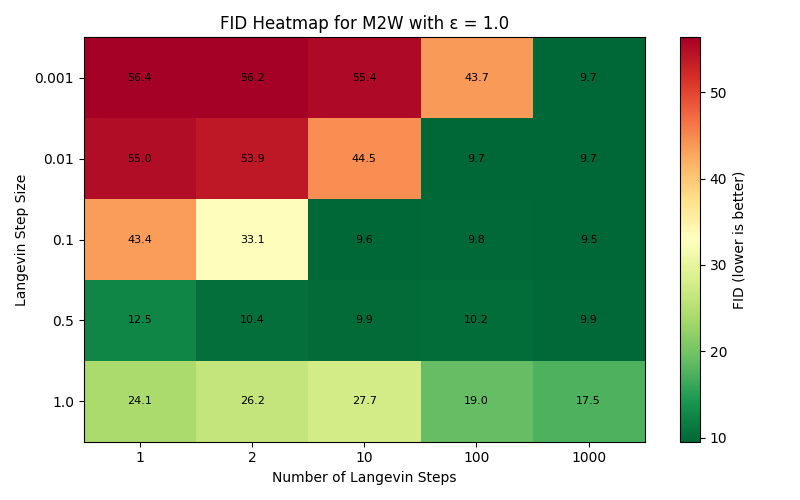}
        \caption{$\varepsilon = 1.0$}
    \end{subfigure}
    \hfill
    \begin{subfigure}{0.48\textwidth}
        \centering
        \includegraphics[width=\linewidth]{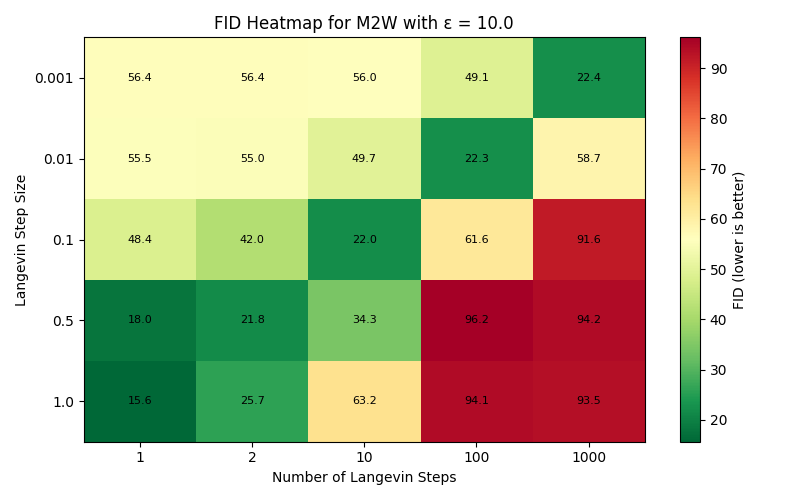}
        \caption{$\varepsilon = 10.0$}
    \end{subfigure}

    \caption{
    FID heatmaps for the \textit{Male}$\rightarrow$\textit{Female} (M$\to$F) unpaired
    image-to-image translation task in the ALAE latent space.
    Each heatmap shows the dependence of FID on the Langevin step size (rows)
    and the number of inference steps (columns) for a fixed value of the
    entropic regularization parameter $\varepsilon$.
    Lower values indicate better performance.
    }
    \label{fig:fid-heatmaps}
\end{figure*}

\end{document}